\documentclass[letterpaper]{article}
\usepackage[margin=1in]{geometry}

\usepackage[utf8]{inputenc} % allow utf-8 input
\usepackage[T1]{fontenc}    % use 8-bit T1 fonts

\usepackage{microtype}
\usepackage{graphicx}
\usepackage{booktabs} % for professional tables

\usepackage{hyperref}

\usepackage{natbib}
\PassOptionsToPackage{numbers, compress}{natbib}

\usepackage{amsmath}
\usepackage{amssymb}
\usepackage{mathtools}
\usepackage{amsthm}

\usepackage[capitalize,noabbrev]{cleveref}

\usepackage{thm-restate}
\usepackage{bbm}
\usepackage{caption}
\usepackage{subcaption}

\usepackage{algorithm, algpseudocode}

\usepackage{commands}
\newcommand{\alg}{\text{Exponentiated O2NC}}     % Our algorithm name

\theoremstyle{plain}
\newtheorem{theorem}{Theorem}[section]
\newtheorem{proposition}[theorem]{Proposition}
\newtheorem{lemma}[theorem]{Lemma}
\newtheorem{corollary}[theorem]{Corollary}
\theoremstyle{definition}
\newtheorem{definition}[theorem]{Definition}
\newtheorem{assumption}[theorem]{Assumption}
\theoremstyle{remark}

\usepackage[textsize=tiny]{todonotes}

\newcommand{\STATE}{\State}
\newcommand{\FOR}{\For}
\newcommand{\ENDFOR}{\EndFor}

\title{Random Scaling and Momentum for Non-smooth Non-convex Optimization}
\author{
  Qinzi Zhang \\
  Boston University \\
  \texttt{qinziz@bu.edu} \\
  \and
  Ashok Cutkosky \\
  Boston University \\
  \texttt{cutkosky@bu.edu} \\
}

\begin{document}

\maketitle

\begin{abstract}
Training neural networks requires optimizing a loss function that may be highly irregular, and in particular neither convex nor smooth. Popular training algorithms are based on stochastic gradient  descent with momentum (SGDM), for which classical analysis applies only if the loss is either convex  or smooth. We show that a very small modification to SGDM closes this gap: simply scale the update at each time point by an exponentially  distributed random scalar. The resulting algorithm achieves optimal convergence guarantees. Intriguingly, this result is not derived by a specific analysis of SGDM: instead, it falls naturally out of a more general framework for converting online convex optimization algorithms to non-convex optimization algorithms.

% We introduce $(c,\epsilon)$-stationary point as a new convergence criterion for non-smooth non-convex stochastic optimization, which, with $c=\epsilon/\delta^2$, is a relaxation of the standard $(\delta,\epsilon)$-stationary point criterion. Additionally, we propose a general framework that converts online convex optimization algorithms to non-smooth optimization algorithms. Applied with an unconstrained variant of online gradient descent, this algorithm recovers the update mechanism of stochastic gradient descent with momentum. Applied with online gradient descent with adaptive learning rate, this algorithm further reduces to an adaptive momentum algorithm. Both algorithms find $(c,\epsilon)$-stationary point within $O(c^{1/2}\epsilon^{-7/2})$ iterations. Remarkably, they automatically achieve the optimal rates of $O(\epsilon^{-4})$ for smooth objectives and $O(\epsilon^{-7/2})$ for second-order smooth objectives.
\end{abstract}

\section{Introduction}

Non-convex optimization algorithms are one of the fundamental tools in modern machine learning, as training neural network models requires optimizing a non-convex loss function. This paper provides a new theoretical framework for building such algorithms. The simplest application of this framework almost exactly recapitulates the standard algorithm used in practice: stochastic gradient descent with momentum (SGDM).
% . Motivated by this, there have been abundant works focused on understanding non-convex optimization from both theoretical and empirical perspectives. 

The goal of any optimization algorithm used to train a neural network is to minimize a potentially non-convex objective function. Formally, given $F:\RR^d\to\RR$, the problem is to solve
\begin{equation*}
    \min_{\x\in\RR^d} F(\x) = \Ex_{z}[f(\x,z)],
\end{equation*}
where $f$ is a stochastic estimator of $F$. In practice, $\x$ denotes the parameters of a neural network model, and $z$ denotes the data point. Following the majority of the literature, we focus on first-order stochastic optimization algorithms, which can only access to the stochastic gradient $\nabla f(\x,z)$ as an estimate of the unknown true gradient $\nabla F(\x)$. We measure the ``cost'' of an algorithm by counting the number of stochastic gradient evaluations it requires to achieve some desired convergence guarantee. We will frequently refer to this count as the number of ``iterations'' employed by the algorithm.

When the objective function is non-convex, finding a global minimum can be intractable. To navigate this complexity, many prior works have imposed various smoothness assumptions on the objective. These include, but not limited to, first-order smoothness \cite{ghadimi2013stochastic, carmon2017convex, arjevani2022lower, carmon2019lower}, second-order smoothness \cite{tripuraneni2018stochastic, carmon2018accelerated, fang2019sharp, arjevani2020second}, and mean-square smoothness \cite{allen2018natasha, fang2018spider, cutkosky2019momentum, arjevani2022lower}. Instead of finding the global minimum, the smoothness conditions allow us to find an $\epsilon$-stationary point $\x$ of $F$ such that $\|\nabla F(\x)\|\le\epsilon$.

The optimal rates for smooth non-convex optimization are now well-understood. When the objective is smooth, stochastic gradient descent (SGD) requires $O(\epsilon^{-4})$ iterations to find $\epsilon$-stationary point, matching the optimal rate \cite{arjevani2019lower}. When $F$ is second-order smooth, a variant of SGD augmented with occasional random perturbations achieves the optimal rate $O(\epsilon^{-7/2})$ \cite{fang2019sharp, arjevani2020second}. Moreover, when $F$ is mean-square smooth, variance-reduction algorithms, such as SPIDER \cite{fang2018spider} and SNVRG \cite{zhou2018stochastic}, achieve the optimal rate $O(\epsilon^{-3})$ \cite{arjevani2019lower}. All of these algorithms can be viewed as variants of SGD.

In addition to these theoretical optimality results, SGD and its variants are also incredibly effective in practice across a wide variety of deep learning tasks. Among these variants, the family of momentum algorithms have become particularly popular \cite{sutskever2013importance, kingma2014adam, you2017scaling, you2019large, cutkosky2019momentum, cutkosky2020momentum, liu2021laprop}. Under smoothness conditions, the momentum algorithms also achieve the same optimal rates.
% ELABORATE: how to draw connections between momentum algorithms and the goal of our paper?

However, modern deep learning models frequently incorporate a range of non-smooth architectures, including elements like ReLU, max pooling, and quantization layers. These components result in a non-smooth optimization objective, violating the fundamental assumption of a vast majority of prior works. Non-smooth optimization is fundamentally more difficult than its smooth counterpart, as in  the worst-case \citet{kornowski2022oracle} show that it is actually  impossible to find a neighborhood around $\epsilon$-stationary. This underscores the need for a tractable convergence criterion in non-smooth non-convex optimization.

One line of research in non-smooth non-convex optimization studies weakly-convex objectives \citep{davis2019stochastic, mai2020convergence}, with a focus on finding $\epsilon$-stationary points of the Moreau envelope of the objectives. It has been demonstrated that various algorithms, including the proximal subgradient method and SGDM, can achieve the optimal rate of $O(\epsilon^{-4})$ for finding an $\epsilon$-stationary point of the Moreau envelope. However, it is important to note that the assumption of weak convexity is crucial for the convergence notion involving the Moreau envelope. Our interest lies in solving non-smooth non-convex optimization without relying on the weak convexity assumption.

To this end, \citet{zhang2020complexity} proposed employing Goldstein stationary points \citep{goldstein1977optimization} as a convergence target in non-smooth non-convex (and non-weakly-convex) optimization. This approach has been widely accepted by recent works studying non-smooth optimization \cite{kornowski2022complexity, lin2022gradientfree, kornowski2023algorithm, cutkosky2023optimal}. Formally, $\x$ is a $(\delta,\epsilon)$-Goldstein stationary point if there exists a random vector $\y$ such that $\E[\y]=\x$, $\|\y-\x\| \le \delta$ almost surely, and $\|\E[\nabla F(\y)]\|\le\epsilon$.\footnote{To be consistent with our proposed definition, we choose to present the definition of $(\delta,\epsilon)$-Goldstein stationary point involving a random vector $\y$. This presentation is equivalent to the original definition \cite{goldstein1977optimization} used in \cite{zhang2020complexity}.} The best-possible rate for finding a $(\delta,\epsilon)$-Goldstein stationary point is $O(\delta^{-1}\epsilon^{-3})$ iterations. This rate was only recently achieved  by \citet{cutkosky2023optimal}, who developed an ``online-to-non-convex conversion'' (O2NC) technique that converts online convex optimization (OCO) algorithms to non-smooth non-convex stochastic optimization algorithms. Building on this background, we will relax the definition of stationarity and extend this O2NC technique to eventually develop a convergence analysis of SGDM in the non-smooth and non-convex setting.
% Notably, applied with online gradient descent, O2NC requires $O(\delta^{-1}\epsilon^{-3})$ for finding $(\delta,\epsilon)$-stationary point, which is the optimal rate.

\subsection{Our Contribution}

In this paper, we introduce a new notion of stationarity for non-smooth non-convex objectives. Our notion is a natural relaxation of the Goldstein stationary point, but will allow for  more flexible algorithm design. Intuitively, the problem with the Goldstein stationary point is that to verify that a point $\x$ is a stationary point, one must evaluate the gradient many times inside a ball of some small radius $\delta$ about $\x$. This means that algorithms finding such points usually make fairly conservative updates to sufficiently explore this ball: in essence, they work by verifying each iterate is \emph{not} close to a stationary point before moving on to the next iterate. Algorithms used in practice do not typically behave this way, and our relaxed definition will not require us to employ such behavior.

% $(c,\epsilon)$-stationary point as an alternative convergence criterion for non-smooth optimization. Formally, we say $\x$ is a $(c,\epsilon)$-stationary point if there exists a random vector $\y$ such that $\E[\y]=\x$, $\E\|\y-\x\|^2\le \epsilon/c$, and $\|\E\nabla F(\y)\|\le \epsilon$. With $c=\epsilon/\delta^2$, this definition aligns closely with the standard $(\delta,\epsilon)$-stationary point except for a key modification: it relaxes the deterministic bound of $\y$, namely $\|\y-\x\|\le \delta$, to a probabilistic bound of its variance, $\E\|\y-\x\|^2 \le \delta^2$. This adjustment expands the scope of algorithms under consideration, accommodating those that may violate the deterministic bound but satisfy the variance bound. Furthermore, similar to $(\delta,\epsilon)$-stationary point, our $(c,\epsilon)$-stationary point criterion automatically reduces to $\epsilon$-stationary point with proper choices of $c$ when the objective is smooth or second-order smooth.

Using our new criterion, we propose a general framework, ``\alg'', that converts OCO algorithms to non-smooth optimization algorithms. This framework is an extension of the  O2NC technique of \citet{cutkosky2023optimal} that distinguishes itself through two significant improvements.

Firstly, the original O2NC method requires the OCO algorithm to constrain all of its iterates to a small ball of radius roughly $\delta\epsilon^2$. 
% to periodically restarting the online learning algorithm every $T$ iterations (where typically $T=\sqrt{N}$ with $N$ denoting the total number of iterations). 
This approach is designed to ensure that the parameters within any period of $\epsilon^{-2}$ iterations remain inside a ball of radius $\delta$. The algorithm then uses these $\epsilon^{-2}$ gradient evaluations inside a ball of radius $\delta$ to check if the current iterate is a stationary point (i.e., if the average gradient has norm less than $\epsilon$). Our new criterion, however, obviates the need for such explicit constraints, intuitively allowing our algorithms to make larger updates when far from a stationary point.
% We demonstrate that, even without periodic resets, it is possible for any given iterate $\x_n$ to be associated with a random vector $\y_n$ whose variance remains controlled. This allows the OCO algorithm to proceed uninterrupted.

Secondly, O2NC does not evaluate gradients at the actual iterates. Instead, gradients are evaluated at an intermediate variable $\w_n$ lying between the two iterates $\x_n$ and $\x_{n+1}$. This conflicts with essentially all practical  algorithms, and moreover imposes an extra memory burden. In contrast, our algorithm evaluates  gradients exactly at each iterate, which simplifies implementation and improves space complexity.
% by setting $s_n$ to an exponential random variable, our algorithm eliminates the requirement for this intermediate state, further improving the memory complexity.

Armed  with this improved framework, we proposed an unconstrained variant of online gradient descent, which is derived from the family of online mirror descent with composite loss. When applied within this algorithm, our framework produces an algorithm that is exactly equal to stochastic gradient descent with momentum (SGDM), subject to an additional random scaling on the update. Notably, it also achieves the optimal rate under our new criterion.
% Next, we consider a standard adaptive learning rate schedule for online gradient descent. Integrated with \alg, this approach evolves into an adaptive momentum algorithm. 

To summarize, this paper has the following contributions:
\begin{itemize}
    \item We introduce a relaxed convergence criterion for non-smooth optimization that recovers all useful properties of Goldstein stationary point.

    \item We propose a modified online-to-non-convex conversion framework that does not require intermediate states.

    \item We apply our new conversion to the most standard OCO algorithm: ``online gradient descent''. The resulting method achieves optimal convergence guarantees as is almost exactly the same as the standard SGDM algorithm. The only difference is that the updates of  SGDM are now scaled by an exponential random variable. This is especially remarkable because the machinery that we employ does not particularly resemble SGDM until it is finally all put together.
    % Further applied with adaptive learning rate, our algorithm achieves adaptive convergence guarantee.
    
    % our algorithm finds a $(c,\epsilon)$-stationary point within $O(c^{1/2}\epsilon^{-7/2})$ iterations. With $c=\epsilon/\delta^2$, this recovers the optimal rate $O(\delta^{-1}\epsilon^{-3})$ under the standard $(\delta,\epsilon)$-stationary point criterion. Notably, this algorithm recovers SGDM, thus providing a theoretical proof for its empirical success.

    % \item Applied with adaptive learning rate, our algorithm again achieves the $O(c^{1/2}\epsilon^{-7/2})$ rate while recovering an adaptive momentum algorithm.
\end{itemize}
\section{Preliminaries}
\label{sec:pre}

\paragraph{Notations}
Bold font $\x$ denotes a vector in $\RR^d$ and $\|\x\|$ denotes its Euclidean norm. We define $B_d(\x,r)=\{\y\in\RR^d:\|\x-\y\|\le r\}$ and sometimes drop the subscript $d$ when the context is clear. We use $[n]$ as an abbreviation for $\{1,2,\ldots,n\}$. We adopt the standard big-$O$ notation, and $f\lesssim g$ denotes $f=O(g)$. $\calP(S)$ denotes the set of all distributions over a measurable set $S$.

\paragraph{Stochastic Optimization}
Given a function $F:\RR^d\to\RR$, $F$ is $G$-Lipschitz if $|F(\x)-F(\y)| \le G\|\x-\y\|, \forall \x,\y$. Equivalently, when $F$ is differentiable, $F$ is $G$-Lipschitz if $\|\nabla F(\x)\|\le G, \forall \x$. $F$ is $H$-smooth if $F$ is differentiable and $\nabla F$ is $H$-Lipschitz; $F$ is $\rho$-second-order-smooth if $F$ is twice differentiable and $\nabla^2 F$ is $\rho$-Lipschitz.

\begin{assumption}
\label{as:optimization}
We assume that our objective function $F:\RR^d\to \RR$ is differentiable and $G$ Lipschitz, and given an initial point $\x_0$, $F(\x_0)-\inf F(\x) \le F^*$ for some known $F^*$. We also assume the stochastic gradient satisfies $\E[\nabla f(\x,z)\,|\,\x]=\nabla F(\x), \E\|\nabla F(\x)-\nabla f(\x,z)\|^2 \le \sigma^2$ for all $\x,z$. Finally, we assume that $F$ is \emph{well-behaved} in the sense of \cite{cutkosky2023optimal}: for any points $\x$ and $\y$, $F(\x)-F(\y) = \int_0^1 \langle \nabla F(\x+t(\y-\x)), \y-x\rangle \ dt$.
% Following the convention in literature, we assume $F$ is $G$-Lipschitz, $F(\x_0)-\inf F(\x) \le F^*$ and $\E[\nabla f(\x,z)\,|\,\x]=\nabla F(\x), \E\|\nabla F(\x)-\nabla f(\x,z)\|^2 \le \sigma^2$ for all $\x,z$.
\end{assumption}

\paragraph{Online Learning}
An online convex optimization (OCO) algorithm is an iterative algorithm that uses the following procedure: in each iteration $n$, the algorithm plays an action $\Delta_n$ and then receives a convex  loss function $\ell_n$ The goal is to minimize the regret w.r.t. some comparator $\u$, defined as
\begin{equation*}
    \regret_n(\u) := \littlesum_{t=1}^n \ell_t(\Delta_t) - \ell_t(\u).
\end{equation*}
The most basic OCO algorithm is online gradient  descent: $\Delta_{n+1} = \Delta_n - \eta \nabla \ell_n(\Delta_n)$, which guarantees $\regret_N(\u) =O(\sqrt{N})$ for appropriate $\eta$. Notably, in OCO we make \emph{no assumptions} about the dynamics of $\ell_n$. They need not be stochastic, and could even be adversarially generated. We will be making use of algorithms that obtain \emph{anytime} regret bounds. That is, for all $n$ and any sequence of $\u_1,\u_2,\dots$, it is possible to bound $\regret_n(\u_n)$ by some appropriate quantities (that may be function of $n$). This is no great burden: almost all online convex optimization algorithms have anytime regret bounds.
% Note that formulation above is technically a simplified description of OCO sometimes called ``online \emph{linear} optimization'' or OLO. 
For readers interested in more details, please refer to \cite{cesa2006prediction, hazan2019introduction, orabona2019modern}.

% \paragraph{Momentum Algorithms}
% We recall the update formula of two popular momentum algorithms, SGDM and Adam. Given stochastic gradient $\g_t=\nabla f(\x_t,z_t)$, SGDM updates $\m_{t+1} = \beta\m_t + (1-\beta)\g_t$ and $\x_{t+1} = \x_t - \eta\m_{t+1}$.\footnote{The original update in \cite{sutskever2013importance} is $\v_{t+1} = \beta\v_t - \mu\g_t$ and $\x_{t+1}=\x_t+\v_{t+1}$. However, substituting $\mu=(1-\beta)\eta$ and $\v_t = \frac{-\mu}{1-\beta}\m_t$ recovers the same update.} Here $\beta$ denotes the momentum constant and $\eta$ denotes the learning rate.

% Similarly, Adam updates $\m_{t+1} = \beta_1\m_t + (1-\beta_1)\g_t$, $\v_{t+1} = \beta_2\v_t + (1-\beta_2)\g_t^2$ and $\x_{t+1} = \x_t - \eta \frac{\m_{t+1}/(1-\beta_1^t)}{\sqrt{\v_{t+1}/(1-\beta_2^t)}}$. Here $\v_t$ is updated coordinate-wise.

\subsection{Non-smooth Optimization}
% \paragraph{Convergence Criterion}

Suppose $F$ is differentiable. $\x$ is an $\epsilon$-stationary point of $F$ if $\|\nabla F(\x)\|\le \epsilon$. This definition is the standard criterion for smooth non-convex optimization.
For non-smooth non-convex optimization, the standard criterion is the following: $\x$ is an $(\delta,\epsilon)$-Goldstein stationary point of $F$ if there exists $S\subset \RR^d$ and $P\in\calP(S)$ such that $\y\sim P$ satisfies $\E[\y]=\x$, $\|\y-\x\|\le\delta$ almost surely, and $\|\E[\nabla F(\y)]\|\le\epsilon$.\footnote{The original definition of $(\delta,\epsilon)$ Goldstein stationary point proposed by \cite{goldstein1977optimization, zhang2020complexity} does not require the condition $\E[\y]=\x$. However, as shown in \cite{cutkosky2023optimal}, this condition allows us to reduce a Goldstein stationary point to an $\epsilon$-stationary point when the loss is second-order smooth. Hence we also keep this condition.} Next, we formally define $(c,\epsilon)$-stationary point, our proposed new criterion for non-smooth optimization.

\begin{definition}
\label{def:stationary-point}
Suppose $F:\RR^d\to\RR$ is differentiable, $\x$ is a $(c,\epsilon)$-stationary point of $F$ if $\|\nabla F(\x)\|_c \le \epsilon$, where
\begin{equation*}
    \|\nabla F(\x)\|_c = \inf_{\substack{S\subset \RR^d \\ \y\sim P\in\calP(S) \\ \E[\y]=\x}} \| \E[\nabla F(\y)]\| + c \cdot \Ex\|\y-\x\|^2.
\end{equation*}
\end{definition}

In other words, if $\x$ is a $(c,\epsilon)$-stationary point, then there exists $S\subset \RR^d, P\in\calP(S)$ such that $\y\sim P$ satisfies $\E[\y]=\x$, $\E\|\y-\x\|^2 \le \epsilon/c$, and $\|\E[\nabla F(\y)]\|\le \epsilon$. To see how this definition is related to the previous $(\epsilon,\delta)$-Goldstein stationary point definition, consider the case when $c=\epsilon/\delta^2$. Then this new definition of $(c,\epsilon)$-stationary point is almost identical to $(\delta,\epsilon)$-Goldstein stationary point, except that it relaxes the constraint from $\|\y-\x\|\le\delta$ to $\E\|\y-\x\|^2\le\delta^2$. 

% ELABORATE why this is good definition

To further motivate this definition, we demonstrate that $(c,\epsilon)$-stationary point retains desirable properties of Goldstein stationary points. Firstly, the following result shows that, similar to Goldstein stationary points, $(c,\epsilon)$-stationary points can also be reduced to first-order stationary points with proper choices of $c$ when the objective is smooth or second-order smooth. 

\begin{restatable}{lemma}{CriterionReduction}
\label{lem:criterion-reduction}
% Assume $\|\nabla F(\x)\|_c\le \epsilon$ for any $c=\epsilon\delta^{-2}$ and $\delta>0$. Then $\|\nabla F(\x)\| \le \epsilon$ if $F$ is $H$-smooth and $\delta=\eps/2H$ or if $F$ is $\rho$-second-order-smooth and $\delta=\sqrt{2\epsilon}/\sqrt{\rho}$.
Suppose $F$ is $H$-smooth. 
If $\|\nabla F(\x)\|_c\le\epsilon$ where $c=H^2\epsilon^{-1}$, then $\|\nabla F(\x)\|\le 2\epsilon$. \\
% If $\|\nabla F(\x)\|_c\le\epsilon$ where $c=\epsilon\delta^{-2}, \delta=\epsilon/2H$, then $\|\nabla F(\x)\|\le 2\epsilon$. \\
Suppose $F$ is $\rho$-second-order-smooth. 
If $\|\nabla F(\x)\|_c\le\epsilon$ where $c=\rho/2$, then $\|\nabla F(\x)\|\le 2\epsilon$.
% If $\|\nabla F(\x)\|_c\le\epsilon$ where $c=\epsilon\delta^{-2}, \delta=\sqrt{2\epsilon}/\sqrt{\rho}$, then $\|\nabla F(\x)\|\le 2\epsilon$.
\end{restatable}

As an immediate implication, suppose an algorithm achieves $O(c^{1/2}\epsilon^{-7/2})$ rate for finding a $(c,\epsilon)$-stationary point. Then Lemma \ref{lem:criterion-reduction} implies that, with $c=O(\epsilon^{-1})$, the algorithm automatically achieves the optimal rate of $O(\epsilon^{-4})$ for smooth objectives \cite{arjevani2019lower}. Similarly, with $c=O(1)$, it achieves the optimal rate of $O(\epsilon^{-7/2})$ for second-order smooth objectives \cite{arjevani2020second}. 

Secondly, we show in the following lemma that $(c,\epsilon)$-stationary points can also be reduced to Goldstein stationary points when the objective is Lipschitz.

\begin{restatable}{lemma}{GoldsteinReduction}
\label{lem:goldstein-reduction}
% Suppose $F$ is $G$-Lipschitz. For any $c,\epsilon,\delta>0$, a $(c,\epsilon)$-stationary point is also a $(\delta,\epsilon')$-Goldstein stationary point where $\epsilon' = (1+\frac{2G}{c\delta^2})\epsilon$.
Suppose $F$ is $G$-Lipschitz. For any $c,\epsilon>0$, if $\x$ is a $(c,\epsilon)$-stationary point, then for any $\delta>0$, there exists $\x'$ such that $\|\x-\x'\|<\delta$ and $\x'$ is a $(2\delta,\epsilon')$-Goldstein stationary point where $\epsilon' = (1+\frac{2G}{c\delta^2})\epsilon$.
\end{restatable}

We defer the full proof of Lemma~\ref{lem:goldstein-reduction} to Appendix~\ref{app:pre}. Here is a brief proof sketch: given a distribution $\y$ that satisfies $(c,\epsilon)$-stationary, we clip $\y$ into $\mathcal{B}(\x,\delta)$ so that the mean of clipped distribution $\hat\y$ is a Goldstein stationary point. Note that our notion of Goldstein stationarity is slightly more stringent than the standard definition (e.g. see \cite{goldstein1977optimization,zhang2020complexity}) because we require $\E[\hat y]=x$, rather than simply requiring $\|\hat y-x\|\le \delta$ with probability 1. If we do not require $\E[\hat y]=\x$, then $\x$ would be itself Goldstein stationary. Moreover, if the original $\y$ that certifies $(c,\epsilon)$ stationarity has finite support, then we can compute the clipped mean $\E[\hat\y]$ in $O(|\mathrm{supp}(\y)|)$ time. In our setting of interest, $\y$ has support $\{\x_1,\ldots,\x_T\}$ where $\x_t$ is the optimizer parameter in each iterate, and $|\mathrm{supp}(\y)|=T$.

\subsection{Online-to-non-convex Conversion}

\begin{algorithm}[t]
\caption{O2NC \cite{cutkosky2023optimal}}
\label{alg:o2nc-original}
\begin{algorithmic}[1]
    \STATE \textbf{Input:} OCO algorithm $\calA$, initial state $\x_0$, parameters $N,K,T\in\NN$ such that $N=KT$.
    \FOR {$n\gets 1,2,\ldots,N$}
        \STATE Receive $\Delta_n$ from $\calA$.
        \STATE Update $\x_n \gets \x_{n-1} + \Delta_n$ and $\w_n\gets\x_{n-1}+s_n\Delta_n$, where $s_n\sim\Unif([0,1])$ i.i.d.
        \STATE Compute $\g_n \gets \nabla f(\x_n, z_n)$.
        \STATE Send loss $\ell_n(\Delta) = \langle \g_n,\Delta\rangle$ to $\calA$.
        
        % \STATE 
        // For output only (update every $T$ iteration):
        \STATE If $n=kT$, compute $\overline\w_k = \frac{1}{T}\sum_{t=0}^{T-1} \w_{n-t}$.
    \ENDFOR
    % \State Compute $\x^k \gets \frac{1}{T} \sum_{n\in I_k} \x_n$ where $I_k=[(k-1)T+1,kT]$.
    \STATE Output $\overline \w \sim \Unif(\{\overline\w_k:k\in[K]\})$.
\end{algorithmic}
\end{algorithm}

Since our algorithm is an extension of the online-to-non-convex conversion (O2NC) technique proposed by \cite{cutkosky2023optimal}, we briefly review the original O2NC algorithm. The pseudocode is outlined in Algorithm \ref{alg:o2nc-original}, with minor adjustments in notations for consistency with our presentation. 

At its essence, O2NC shifts the challenge of optimizing a non-convex and non-smooth objective into minimizing regret. The intuition is as follows. By adding a uniform perturbation $s_n\in[0,1]$, $\langle \nabla f(\x_{n-1}+s_n\Delta_n,z_n), \Delta_n\rangle = \langle \g_n,\Delta_n\rangle$ is an unbiased estimator of $F(\x_n)-F(\x_{n-1})$, effectively capturing the ``training progress''. Consequently, by minimizing the regret, which is equivalent to minimizing $\sum_{n=1}^N\langle \g_n,\Delta_n\rangle$, the algorithm automatically identifies the most effective update step $\Delta_n$.

\subsection{Paper Organization}

In Section \ref{sec:alg}, we present the general online-to-non-convex framework, \alg. We  first explain the intuitions behind the algorithm design, and then we provide the convergence analysis in Section \ref{sec:alg-conv}. In Section \ref{sec:sgdm}, we provide an explicit instantiation of our framework, and see that the resulting algorithm is essentially the standard SGDM.
% In section \ref{sec:adam}, we further apply adaptive learning rate.
In Section \ref{sec:lower}, we present a lower bound for finding $(c,\epsilon)$-stationary point.
In Section \ref{sec:experiment}, we present empirical evaluations.
\section{Exponentiated Online-to-non-convex}
\label{sec:alg}

\begin{algorithm}[t]
\caption{Exponentiated O2NC}
\label{alg:O2NC-exp-avg}
\begin{algorithmic}[1]
    \STATE \textbf{Input:} OCO algorithm $\calA$, initial state $\x_0$, parameters $N\in\NN, \beta\in (0,1)$, regularizers $\calR_n(\Delta)$.
    \FOR {$n\gets 1,2,\ldots,N$}
        \STATE Receive $\Delta_n$ from $\calA$.
        \STATE Update $\x_n \gets \x_{n-1} + s_n\Delta_n$, where $s_n\sim\Expo(1)$ i.i.d.
        \STATE Compute $\g_n \gets \nabla f(\x_n, z_n)$.
        \STATE Send loss $\ell_n(\Delta) = \langle \beta^{-n}\g_n, \Delta\rangle + \calR_n(\Delta)$ to $\calA$.
        
        % \STATE 
        // For output only (does \textit{not} affect training):
        \STATE Update $\overline\x_n = \frac{\beta-\beta^n}{1-\beta^n}\overline\x_{n-1} + \frac{1-\beta}{1-\beta^n}\x_n$. \\
        Equivalently, $\overline \x_n = \sum_{t=1}^n \beta^{n-t}\x_t\cdot \frac{1-\beta}{1-\beta^n}$.
    \ENDFOR
    % \State Compute $\x^k \gets \frac{1}{T} \sum_{n\in I_k} \x_n$ where $I_k=[(k-1)T+1,kT]$.
    \STATE Output $\overline \x \sim \Unif(\{\overline\x_n:n\in[N]\})$.
\end{algorithmic}
\end{algorithm}

% \subsection{Algorithm Design}
% \label{sec:alg-design}

% In this section, we introduce a comprehensive framework designed to transform any OCO algorithm with sublinear regret into a stochastic optimization algorithm suitable for non-smooth non-convex losses. This framework is an adaptation of the Online-to-non-convex Conversion (O2NC) algorithm \cite{cutkosky2023optimal}. We have implemented several critical modifications to more effectively recover momentum algorithms. The specifics of each modification are elaborated in Section \ref{sec:alg-design}.

In this section, we present our improved online-to-non-convex framework, \alg, and explain the key techniques we employed to improve the algorithm. The pseudocode is presented in Algorithm \ref{alg:O2NC-exp-avg}.

\paragraph{Random Scaling}
% One notable difference between \alg \ and standard O2NC lies in the selection of the scaling random variable $s_n$. In O2NC, $s_n$ is chosen to be uniform, $s_n\sim \Unif([0,1])$, based on
% \begin{align*}
%     \E[F(\x+\Delta) - F(\x)] = \Ex_{s\sim\Unif([0,1])}[\langle \x+s\Delta), \Delta\rangle].
% \end{align*}
% Given $\x_n=\x_{n-1}+\Delta_n$ and $\w_n=\x_{n-1}+s_n\Delta_n$, it follows that $\langle \nabla F(\w_n),\Delta_n\rangle$ is an unbiased estimator of $F(\x_n)-F(\x_{n-1})$. However, this approach introduces additional memory burden, as it requires  to compute the gradient at the intermediate variable $\w_n$ lying between $\x_{n-1}$ and $\x_n$, rather than directly at the iterate $\x_n$.
% To address this limitation, we choose an exponential random variable as the scaling factor, namely $s_n \sim \Expo(1)$. The reason for this choice is as follows.

One notable feature of Algorithm \ref{alg:O2NC-exp-avg} is that the update $\Delta_n$ is scaled by an exponential random variable $s_n$. 
Formally, we have the following result:

\begin{restatable}{lemma}{ExponentialScaling}
\label{lem:exp-scaling}
Let $s\sim\mathrm{Exp}(\lambda)$ for some $\lambda>0$, then
\begin{equation*}
    \Ex_s[F(\x+s\Delta) - F(\x)] = \Ex_s[\langle \nabla F(\x+s\Delta), \Delta\rangle] / \lambda.
\end{equation*}
\end{restatable}

In Algorihtm~\ref{alg:O2NC-exp-avg}, we set $s_n\sim\mathrm{Exp}(1)$ and then define $\x_n = \x_{n-1} + s_n\Delta_n$. Thus, Lemma \ref{lem:exp-scaling} implies that 
% $\E[F(\x_n) - F(\x_{n-1})] = \E\langle \nabla F(\x_n), \Delta_n\rangle$. 
\begin{align*}
    \E[F(\x_n) - F(\x_{n-1})] 
    &= \E\langle \nabla F(\x_n), \Delta_n\rangle \\
    &= \E\langle \nabla F(\x_n), \x_n-\x_{n-1}\rangle.
\end{align*}
In other words, we can estimate the ``training progress'' $F(\x_n)-F(\x_{n-1})$ by directly computing the stochastic gradient at iterate $\x_n$. By exploiting favorable properties of the exponential distribution, we dispense with the need for the ``auxiliary point'' $\w_n $ employed by O2NC.
% In contrast to the previous scenario where $s_n$ is uniform, this new approach allows us to compute the gradient directly at $\x_n$ to estimate the training progress $F(\x_n)-F(\x_{n-1})$.

We'd like to highlight the significance of this result. The vast majority of smooth non-convex optimization analysis depends on the assumption that $F(\x)$ is locally linear, namely $F(\x_n)-F(\x_{n-1}) \approx \langle \nabla F(\x_n), \x_n-\x_{n-1}\rangle$. Under various smoothness assumptions, the error in this approximation can be controlled via bounds on the remainder in a Taylor series. For example, when $F$ is smooth, then $F(\x_n) - F(\x_{n-1}) = \langle F(\x_n),\x_n-\x_{n-1}\rangle + O(\|\x_n-\x_{n-1}\|^2)$. However, since smoothness is necessary for bounding Taylor approximation error, such analysis technique is inapplicable in non-smooth optimization. In contrast, by scaling an exponential random variable to the update, we directly establish a linear equation that $\E[F(\x_n)-F(\x_{n-1})]=\E\langle \nabla F(\x_n),\x_n-\x_{n-1}\rangle$, effectively eliminating any additional error that Taylor approximation might incur. 

A randomized approach such as ours is also recommended in the recent findings by \citet{jordan2023deterministic}, who demonstrated that randomization is \emph{necessary} for achieving a dimension-free rate in non-smooth optimization. In particular, any deterministic algorithm suffers an additional dimension dependence of $\Omega(d)$. 

Although employing exponential random scaling might seem unconventional, we justify this scaling by noting that $s_n\sim\Expo(1)$ satisfies $\E[s_n]=1$ and $\P\{s_n \ge t\} = \exp(-t)$ (in particular, $\P\{s_n \le 5\} \ge 0.99$). In other words, with high probability, the scaling factor behaves like a constant scaling on the update. To corroborate the efficacy of random scaling, we have conducted a series of empirical tests, the details of which are discussed in Section \ref{sec:experiment}.

\paragraph{Exponentiated and Regularized Losses}
The most significant feature of Exponentiated O2NC (and from which it derives its name) is the loss function: $\ell_n(\Delta) = \langle \beta^{-n}\g_n, \Delta\rangle + \calR_n(\Delta)$. This loss consists of two parts: intuitively, the exponentially upweighted linear loss $\langle \beta^{-n}\g_n,\Delta\rangle$ measures the ``training progress'' $F(\x_n)-F(\x_{n-1})$ (as discussed in Lemma \ref{lem:exp-scaling}), and $\calR_n(\Delta)$ serves as an stabilizer that prevents the iterates from irregular behaviors. We will elaborate the role of each component later.
To illustrate how Exponential O2NC works, let $\u_n$ be the optimal choice of $\Delta_n$ in hindsight. Then by minimizing the regret $\regret_n(\u_n)$ w.r.t. $\u_n$, Algorithm \ref{alg:O2NC-exp-avg} automatically chooses the best possible update $\Delta_n$ that is closest to $\u_n$.

\paragraph{Exponentially Weighted Gradients}
% The key modification distinguishing Exponentiated O2NC from standard O2NC (and the feature from which it derives its name) is the selection of the loss function. While O2NC uses linear loss $\langle\g_n,\Delta\rangle$, we adopt an exponentiated and regularized loss: $\ell_n(\Delta) = \langle \beta^{-n}\g_n, \Delta\rangle + \calR_n(\Delta)$. 
For now, set aside the regularizer $\calR_n$ and focus on the linear term $\langle\beta^{-n}\g_n,\Delta\rangle$. To provide an intuition why we upweight the gradient by an exponential factor $\beta^{-n}$, we provide a brief overview for the convergence analysis of our algorithm. For simplicity of illustration, we assume $\g_n = \nabla F(\x_n)$ and $\calR_n=0$.

Let $S_n=\{\x_t\}_{t=1}^n$ and let $\y_n$ be distributed over $S_n$ such that $\P\{\y_n=\x_t\} = p_{n,t} := \beta^{n-t}\cdot \frac{1-\beta}{1-\beta^n}$. Our strategy will be to show that this set $S_n$ and random variable $\y_n$ satisfy the conditions to make $\overline \x_n$ a $(c,\epsilon)$ stationary point  where $\overline\x_n$ is defined in Algorithm \ref{alg:O2NC-exp-avg}. To start, note that this distribution satisfies $\overline\x_n=\E[\y_n]$. Next, since there is always non-zero probability that $\y_n=\x_1$, it's not possible to obtain a deterministic bound of $\|\y_n-\overline\x_n\|\le \delta$ for some small $\delta$ (as would be required if we were trying to establish $(\delta,\epsilon)$ Goldstein stationarity). However, since $\y_n$ is exponentially more likely to be a later iterate (close to $\x_n$) than an early iterate (far from $\x_n$), intuitively $\E\|\y_n-\overline\x_n\|^2$ should not be too big. Formalizing this intuition forms a substantial part of our analysis.

% In fact, we will see in a moment that $\E\|\y_n-\overline\x_n\|\ll \max\|\y_n-\overline\x_n\|$. 

In the convergence analysis, we will show $\overline\x$ is a $(c,\epsilon)$-stationary point by bounding $\|\nabla F(\overline\x_n)\|_c$ (defined in Definition \ref{def:stationary-point}) for all $n$. The condition $\E[\y_n]=\overline\x_n$ is already satisfied by construction of $\y_n$, and it remains to bound the expected gradient $\|\E[\nabla F(\y_n)]\|$ and the variance $\E\|\y_n-\overline\x_n\|^2$. While the regularizer $\calR_n$ is imposed to control the variance, the exponentiated gradients is employed to bound the expected gradient. 
In particular, this is achieved by reducing the difficult task of minimizing the expected gradient of a non-smooth non-convex objective to a relatively easier (and very heavily studied) one: minimizing the regret w.r.t. exponentiated losses $\ell_t(\Delta) = \langle \beta^{-t}\g_t,\Delta\rangle$.
To elaborate further, let's consider a simplified illustration as follows.
Recall that $p_{n,t} = \beta^{n-t}\cdot \frac{1-\beta}{1-\beta^n}$. By construction of $\y_n$,
\begin{align*}
    \E[\nabla F(\y_n)]
    &= \littlesum_{t=1}^n p_{n,t} \nabla F(\x_t).
\end{align*}
Next, for each $n\in[N]$, we define
\begin{align}
    \u_n = -D\frac{\sum_{t=1}^n p_{n,t} \nabla F(\x_t)}{\|\sum_{t=1}^n p_{n,t}\nabla F(\x_t)\|}
    \label{eq:eo2nc-u}
\end{align}
for some $D$ to be specified later. As a remark, $\u_n$ minimizes $\langle \E[\nabla F(\y_n)], \Delta\rangle$ over all possible $\Delta$ such that $\|\Delta\|=D$, therefore representing the optimal update in iterate $n$ that leads to the fastest convergence.

With $\u_n$ defined in \eqref{eq:eo2nc-u}, it follows that
\begin{align*}
    \frac{1}{D}\sum_{t=1}^n p_{n,t} \langle \nabla F(\x_t), -\u_n\rangle 
    &= \left\| \sum_{t=1}^n p_{n,t} \nabla F(\x_t) \right\| \\
    &= \|\E [\nabla F(\y_n)]\|.
\end{align*}
Recall that we assume $\g_t=\nabla F(\x_t)$ for simplicity. Moreover, in later convergence analysis, we will carefully show that $\sum_{n=1}^N\sum_{t=1}^n p_{n,t} \langle\nabla F(\x_t),-\Delta_t\rangle \lesssim 1-\beta$ (see Appendix \ref{app:alg}). Consequently,
\begin{align*}
    &\frac{1}{N} \sum_{n=1}^N \|\E\nabla F(\y_n)\| \\
    &= \frac{1}{DN} \sum_{n=1}^N \sum_{t=1}^n p_{n,t} \langle \nabla F(\x_t), \Delta_t-\u_n\rangle \\
    &\quad - \frac{1}{DN} \sum_{n=1}^N \sum_{t=1}^n p_{n,t}\langle \nabla F(\x_t),\Delta_t\rangle \\
    &\lesssim \frac{1-\beta}{DN} \left(1+\sum_{n=1}^N \beta^n 
    \regret_n(\u_n)\right).
    % \sum_{t=1}^n \langle \beta^{-t}\g_t, \Delta_t-\u_n\rangle.
\end{align*}
Here $\regret_n(\u_n) = \sum_{t=1}^n \langle \beta^{-t}\g_t,\Delta_t-\u_n\rangle$ denotes the regret w.r.t. the exponentiated losses $\ell_t(\Delta) = \langle \beta^{-t}\g_t,\Delta\rangle$ for $t=1,\ldots,n$ (assuming $\calR_n=0$) and comparator $\u_n$ defined in \eqref{eq:eo2nc-u}. Notably, the expected gradient is now bounded by the weighted average of the sequence of static regrets, $\regret_n(\u_n)$. Consequently, a good OCO algorithm that effectively minimizes the regret is closely aligned with our goal of minimizing the expected gradient.

\paragraph{Variance Regularization}
As aforementioned, we impose the regularizer $\calR_n(\Delta) = \frac{\mu_n}{2}\|\Delta\|^2$ to control the variance $\E\|\y_n-\overline\x_n\|^2$. Formally, the following result establishes a reduction from bounding variance to bounding the norm of $\Delta_t$, thus motivating the choice of the regularizer.

\begin{restatable}{lemma}{VarianceToRegularizer}
\label{lem:variance-regularizer}
For any $\beta\in(0,1)$, 
\begin{align*}
    \Ex_s\sum_{n=1}^N \Ex_{\y_n} \|\y_n-\overline\x_n\|^2 \le \sum_{n=1}^N \frac{12}{(1-\beta)^2}\|\Delta_n\|^2.
\end{align*}
\end{restatable}

This suggests that bounding $\|\Delta_n\|^2$ is sufficient to bound the variance of $\y_n$. Therefore, we impose the regularizer $\calR_n(\Delta) = \frac{\mu_n}{2}\|\Delta\|^2$, for some constant $\mu_n$ to be determined later, to ensure that $\|\Delta_n\|^2$ remains small, effectively controlling the variance of $\y_n$. 

Furthermore, we'd like to highlight that Lemma \ref{lem:variance-regularizer} provides a strictly better bound on the variance of $\y_n$ compared to the possible maximum deviation $\max\|\y_n-\overline\x_n\|$. For illustration, assume $\Delta_t$'s are orthonormal, then $\max\|\y_N-\overline\x_N\| \approx \|\x_1-\x_N\| = O(N)$. On the other hand, Lemma \ref{lem:variance-regularizer} implies that for $n\sim\Unif([N])$, $\E_n[\Var(\y_n)] = O(\frac{1}{(1-\beta)^2})$. In particular, we will show that $1-\beta= N^{-1/2}$ when the objective is smooth. Consequently, $\E\|\y_n-\overline\x_n\| = O(\sqrt{N})$, which is strictly tighter than the deterministic bound of $\max\|\y_N-\overline\x_N\| = O(N)$.

This further motivates why we choose this specific distribution for $\y_n$: the algorithm does not need to be conservative all the time and can occasionally make relatively large step, breaking the deterministic constraint that $\|\y_n-\overline\x_n\|\le\delta$, while still satisfying $\Var(\y_n)\le\delta^2$.

% \paragraph{Unconstrained OCO}
% EITHER HERE OR LATER IN SEC. 4

\subsection{Convergence Analysis}
\label{sec:alg-conv}

Now we present the main convergence theorem of Algorithm \ref{alg:O2NC-exp-avg}. This is a very general theorem, and we will prove the convergence bound of a more specific algorithm (Theorem \ref{thm:sgdm}) based on this result. A more formally stated version of this theorem and its proof can be found in Appendix \ref{app:alg}.  

\begin{theorem}
\label{thm:o2nc}
Follow Assumption \ref{as:optimization}. Let $\regret_n(\u_n)$ denote the regret w.r.t. $\ell_t(\Delta)=\langle\beta^{-t}\g_t,\Delta\rangle + \calR_t(\Delta)$ for $t=1,\ldots,n$ and comparator $\u_n$ defined in \eqref{eq:eo2nc-u}. Define $\calR_t(\Delta) = \frac{\mu_t}{2}\|\Delta\|^2$, $\mu_t = \frac{24cD}{\alpha^2}\beta^{-t}$, and $\alpha=1-\beta$, then
\begin{align*}
    &\E \|\nabla F(\overline\x)\|_c
    \lesssim \frac{F^*}{DN} + \frac{G+\sigma}{\alpha N} + \sigma\sqrt{\alpha} + \frac{c D^2}{\alpha^2} \\
    & + \E\frac{\beta^{N+1}\regret_N(\u_N) + \alpha \sum_{n=1}^N \beta^n \regret_n(\u_n)}{DN}.
\end{align*}
\end{theorem}

Here the second line denotes the weighted average of the sequence of static regrets, $\regret_n(\w_n)$, w.r.t. the exponentiated and regularized loss $\ell_t(\Delta) = \langle \beta^{-t}\g_t,\Delta\rangle$ and comparator $\u_n$ defined in \eqref{eq:eo2nc-u}, as we discussed earlier. To see an immediate implication of Theorem \ref{thm:o2nc}, assume the average regret is no larger than the terms in the first line. Then by proper tuning $D=\frac{1}{\sqrt{\alpha} N}$ and $\alpha=\max\{\frac{1}{N^{2/3}}, \frac{c^{2/7}}{N^{4/7}}\}$, we have $\E\|\nabla F(\overline\w)\|_c \lesssim \frac{1}{N^{1/3}} + \frac{c^{1/7}}{N^{2/7}}$.

% To see a quick example, consider projected online gradient descent (OGD) \cite{zinkevich2003online}, which updates $\Delta_{t+1} = \Pi_D(\Delta_t - \eta\nabla\ell_t(\Delta_t))$ 
\section{Recovering SGDM: \\ {\alg} with OMD}
\label{sec:sgdm}

In the previous sections, we have shown that Exponentiated O2NC can convert any OCO algorithm into a non-convex optimization algorithm in such a way that small regret bounds transform into convergence guarantees. So, the natural next step is to instantiate \alg\ with some particular OCO algorithm. In this section we carry out this task and discover that the resulting method not only achieves optimal convergence guarantees, but is also \emph{nearly identical} to the standard SGDM optimization algorithm!

The OCO algorithm we will use to instantiate Exponentiated O2NC is a simple variant of ``online mirror descent'' (OMD) \cite{beck2003mirror}, which a standard OCO algorithm. However, typical OMD analysis involves clipping the outputs $\Delta_n$ to lie in some pre-specified constraint set. We instead employ a minor modification to the standard algorithm to obviate the need for such clipping. 
% This algorithm works for any arbitrary sequence of gradients $\tilde\g_t$ for $t=1,\ldots,n$. In the context of Exponentiated O2NC, $\tilde\g_t = \beta^{-t}\g_t$.

Inspired by \cite{duchi2010composite}, we choose our OCO algorithm from the family of Online Mirror Descent (OMD) with composite loss.  Given a sequence of gradients $\tilde\g_t:=\beta^{-t} \g_t$ and convex functions $\psi_t(\Delta), \calR_t(\Delta), \phi_t(\Delta)$, OMD with composite loss defines $\Delta_{t+1}$ as:
\begin{equation*}
    % \Delta_{t+1} = 
    \argmin_\Delta \langle\tilde\g_t, \Delta\rangle + D_{\psi_t}(\Delta,\Delta_t) + \underbrace{\calR_{t+1}(\Delta) + \phi_t(\Delta)}_{\text{composite loss}}.
    \label{eq:omd-update}
\end{equation*}
Here $D_{\psi_t}$ denotes the Bregman divergence of $\psi_t$, and $\calR_{t+1}(\Delta)+\phi_t(\Delta)$ denotes the composite loss. 
% Without $\phi_t$, this update formula simplifies back to the conventional OMD update. The introduction of composite loss serves as a counterbalance to the potentially large norms of $\Delta_t$, effectively acting as a norm regularizer. 
The composite loss consists of two components. Firstly, $\calR_{t+1}(\Delta)=\frac{\mu_{t+1}}{2}\|\Delta\|^2$ controls the variance of $\y_n$, as discussed in Section \ref{sec:alg}. Secondly, OMD is known to struggle under unconstrained domain setting \cite{orabona2016scale}, but this can be fixed with proper regularization, as done in \citet{fang2021online} (implicitly), and \citet{jacobsen2022parameter} (explicitly). Following a similar approach, we set $\phi_t(\Delta)=(\frac{1}{\eta_{t+1}}-\frac{1}{\eta_t})\|\Delta\|^2$ to prevent the norm of $\|\Delta\|$ from being too large.

% In particular, we choose $\psi_t(\Delta) = \frac{1}{2\eta_t}\|\Delta\|^2$ where $\eta_t$ denotes some learning rate such that $0<\eta_{t+1}\le \eta_t$. Moreover, with slight abuse of notation, we define the composite loss as $\calR_{t+1}(\Delta) + \phi_t(\Delta)$. Here $\calR_{t+1}(\Delta)=\frac{\mu_{t+1}}{2}\|\Delta\|^2$ functions as the variance regularizer of \alg (as detailed in Section \ref{sec:alg-design}); $\phi_t(\Delta)=(\frac{1}{\eta_{t+1}}-\frac{1}{\eta_t})\|\Delta\|^2$ serves as the norm regularizer for the unconstrained domain.

With $\psi_t(\Delta) = \frac{1}{2\eta_t}\|\Delta\|^2$ where $0<\eta_{t+1}\le \eta_t$, Theorem \ref{thm:OGD-brief} provides a regret bound for this specific OCO algorithm.

\begin{theorem}
\label{thm:OGD-brief}
Let $\Delta_1=0$ and $\Delta_{t+1} = \argmin_\Delta \langle\tilde\g_t,\Delta\rangle + \frac{1}{2\eta_t}\|\Delta-\Delta_t\|^2 + \frac{\mu_{t+1}}{2}\|\Delta\|^2 + (\frac{1}{\eta_{t+1}}-\frac{1}{\eta_t})\|\Delta\|^2$. Then
\begin{align*}
    &\sum_{t=1}^n \langle \tilde\g_t, \Delta_t-\u\rangle + \calR_t(\Delta_t) - \calR_t(\u) \\
    &\le \left(\frac{2}{\eta_{n+1}} + \frac{\mu_{n+1}}{2}\right)\|\u\|^2 + \frac{1}{2}\sum_{t=1}^n \eta_t\|\tilde\g_t\|^2.
\end{align*}
\end{theorem}

% Before we draw connections with \alg, we'd like to highlight some key takeaways from Theorem \ref{thm:OGD-brief}. 
Note that the implicit OMD update described in Theorem \ref{thm:OGD-brief} can be explicitly represented as follows:
\begin{equation}
    \Delta_{t+1} = \frac{\Delta_t-\eta_t\tilde\g_t}{1+\eta_t\mu_{t+1}+\eta_t(\frac{1}{\eta_{t+1}}-\frac{1}{\eta_t})}.
    \label{eq:ogd-explicit}
\end{equation}
When $\calR_t=0$ (implying $\mu_t=0$), the update formula in \eqref{eq:ogd-explicit} simplifies to an approximation of \emph{online gradient descent} \cite{zinkevich2003online}, albeit with an additional scaling.
% In other words, Theorem \ref{thm:OGD-brief} gives a regret bound of an unconstrained variant of online gradient descent with varying step-size.

\subsection{Reduction of \alg}

Upon substituting $\tilde\g_t = \beta^{-t}\g_t$ where $\g_t=\nabla f(\x_t,z_t)$, Theorem \ref{thm:OGD-brief} provides a regret bound for $\regret_n(\u_n)$ in the convergence bound in Theorem \ref{thm:o2nc}. Consequently, we can bound $\E\|\nabla F(\overline\x)\|_c$ for \alg with the unconstrained variant of OMD as the OCO subroutine (with update formula described in \eqref{eq:ogd-explicit}). Formally, we have the following result:

\begin{restatable}{theorem}{SGDM}
\label{thm:sgdm}
Follow Assumption \ref{as:optimization} and consider any $c>0$. Let $\Delta_1=0$ and update $\Delta_t$ by
\begin{equation*}
    \Delta_{t+1} = \frac{\Delta_t - \eta_t\beta^{-t}\g_t}{1+\eta_t\mu_{t+1}+\eta_t(\frac{1}{\eta_{t+1}}-\frac{1}{\eta_t})}.
\end{equation*}
Let $\mu_t=\beta^{-t}\mu$, $\eta_t=\beta^t\eta$, $\beta=1-\alpha$, $\mu=\frac{24F^*c}{(G+\sigma)\alpha^{5/2}N}$, $\eta=\frac{2F^*}{(G+\sigma)^2N}$, $\alpha=\max\{N^{-2/3}, \frac{(F^*)^{4/7}c^{2/7}}{(G+\sigma)^{6/7}N^{4/7}}\}$. Then for $N$ large enough such that $\alpha \le \frac{1}{2}$,
\begin{align*}
    \E\|\nabla F(\overline\x)\|_c
    &\lesssim \frac{G+\sigma}{N^{1/3}} + \frac{(F^*)^{2/7}(G+\sigma)^{4/7}c^{1/7}}{N^{2/7}}.
\end{align*}
\end{restatable}

As an immediate implication, upon solving $\E\|\nabla F(\overline\x)\|_c \le \epsilon$ for $N$, we conclude that Algorithm \ref{alg:O2NC-exp-avg} instantiated with unconstrained OGD finds $(c,\epsilon)$-stationary point within $N=O(\max\{(G+\sigma)^3\epsilon^{-3}, F^*(G+\sigma)^2c^{1/2}\epsilon^{-7/2}\})$ iterations. 
% When $c\ge\epsilon$. 
Moreover, in Section \ref{sec:lower} we will show that this rate is optimal.

Furthermore, as discussed in Section \ref{sec:pre}, with $c=O(\epsilon^{-1})$, this algorithm achieves the optimal rate of $O(\epsilon^{-4})$ when $F$ is smooth; with $c=O(1)$, this algorithm also achieves the optimal rate of $O(\epsilon^{-7/2})$ when $F$ is second-order smooth. Remarkably, these optimal rates automatically follows from the reduction from $(c,\epsilon)$-stationary point to $\epsilon$-stationary point (see Lemma \ref{lem:criterion-reduction}), and neither the algorithm nor the analysis is modified to achieve these rates.
% with $c=\epsilon/\delta^2$, the definition of $(c,\epsilon)$-stationary point is a relaxation of the standard $(\delta,\epsilon)$ Goldstein stationary point.
% by relaxing a deterministic constraint, $\|\y_n-\overline\x_n\| \le \delta$, to a probabilistic one, $\E\|\y_n-\overline\x_n\|^2 \le \delta^2$. 
% Moreover, we can reduce $(c,\epsilon)$-stationary point to $\epsilon$-stationary point by setting $\delta=O(\epsilon)$ when the loss is smooth or $\delta=O(\sqrt{\epsilon})$ when the loss is second-order smooth.
% Consequently, upon substituting $c=\epsilon/\delta^2$ into Theorem \ref{thm:sgdm}, $\E\|\nabla F(\overline\x)\|_c \le \epsilon$ with $N\gtrsim \max\{\frac{(G+\sigma)^3}{\epsilon^3}, \frac{F^*(G+\sigma)^2}{\delta\epsilon^3}\}$. In the scenario where $\delta=O(1)$, this oracle complexity matches the optimal rate achieved by standard O2NC \cite{cutkosky2023optimal} under the $(\delta,\epsilon)$ Goldstein stationary point criterion. Moreover, upon substituting $\delta=O(\epsilon)$, we attain the optimal rate $O(\epsilon^{-4})$ for smooth losses; similarly, with $\delta=O(\sqrt{\epsilon})$, we attain the optimal rate $O(\epsilon^{7/2})$ for second-order smooth losses.

\subsection{Unraveling the update to discover SGDM}
Furthermore, upon substituting the definition of $\eta_t,\mu_t$ (and neglecting constants $G,\sigma,F^*$), the update in Theorem \ref{thm:sgdm} can be rewritten as
\begin{equation*}
    \Delta_{t+1} = \frac{\Delta_t - \eta\g_t}{1 + \frac{1}{\beta}(\eta\mu + \alpha)}
\end{equation*}
Let $\Delta_t = -\frac{\beta\eta}{\eta\mu+\alpha}\m_t$, then we can rewrite the update of \alg \ with OGD as follows:
\begin{align}
    &\m_{t+1} = \frac{\beta}{1+\eta\mu} \m_t + \frac{\alpha + \eta\mu}{1+\eta\mu} \g_t, \notag\\
    &\x_{t+1} = \x_t - s_{n+1}\cdot \frac{\beta\eta}{\eta\mu+\alpha} \m_{t+1}.
    \label{eq:sgdm-1}
\end{align}
Remarkably, this update formula recovers the standard SGDM update, with the slight modification of an additional exponential random variable $s_{n+1}$: let $\tilde\beta = \frac{\beta}{1+\eta\mu}$, which denotes the effective momentum constant, and let $\tilde\eta = \frac{\beta\eta}{\eta\mu+\alpha}$ be the effective learning rate, then \eqref{eq:sgdm-1} becomes
\begin{align}
    &\m_{t+1} = \tilde\beta \m_t + (1-\tilde\beta)\g_t, \notag\\
    &\x_{t+1} = \x_t - s_{t+1}\cdot \tilde\eta \m_{t+1}.
    \label{eq:sgdm-1}
\end{align}

\paragraph{Smooth case}
As discussed earlier, when $F$ is smooth, we set $c=O(\epsilon^{-1})$ to recover the optimal rate $N=O(\epsilon^{-4})$. This implies $c=O(N^{1/4})$. Consequently, we can check the parameters defined in Theorem \ref{thm:sgdm} have order $\alpha = O(N^{-1/2})$, $\eta = O(N^{-1})$, and $\mu=O(N^{1/2})$ (note that $\eta\mu\approx \alpha$). Therefore, the effective momentum constant is roughly $\tilde\beta \approx 1-\frac{1}{\sqrt{N}}$, and the effective learning rate is roughly $\tilde\eta\approx \frac{1}{\sqrt{N}}$. Interestingly, these values align with prior works \cite{cutkosky2020momentum}.

\begin{figure*}[ht]
    \vskip 0.2in
    \centering
    \begin{subfigure}[b]{0.3\linewidth}
        \includegraphics[width=\linewidth]{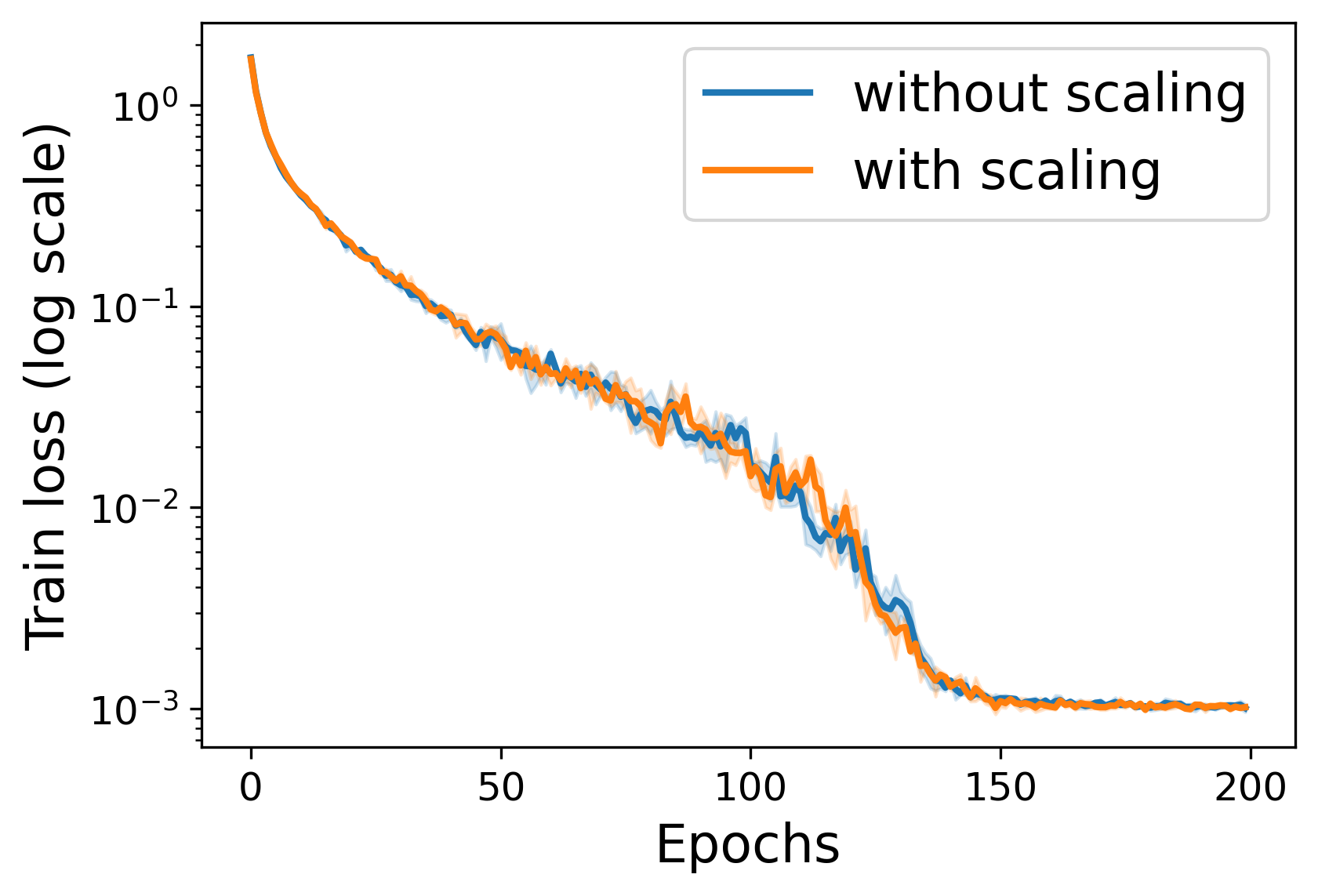}
        \caption{Train Loss}
        \label{fig:1a}
    \end{subfigure}
    \hfill % Optional: add some horizontal spacing
    \begin{subfigure}[b]{0.3\linewidth}
        \includegraphics[width=\linewidth]{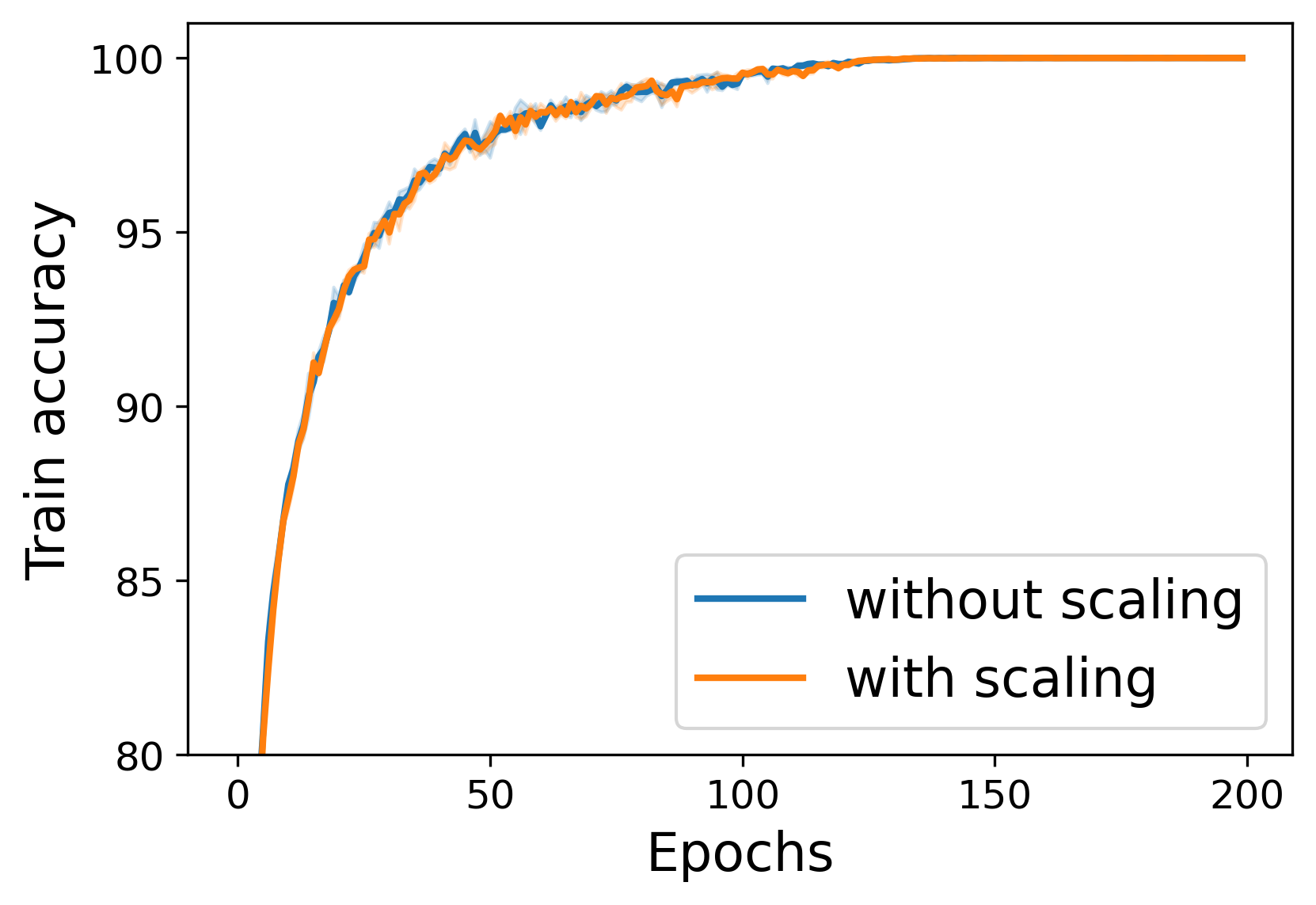}
        \caption{Train Accuracy}
        \label{fig:1b}
    \end{subfigure}
    \hfill % Optional: add some horizontal spacing
    \begin{subfigure}[b]{0.3\linewidth}
        \includegraphics[width=\linewidth]{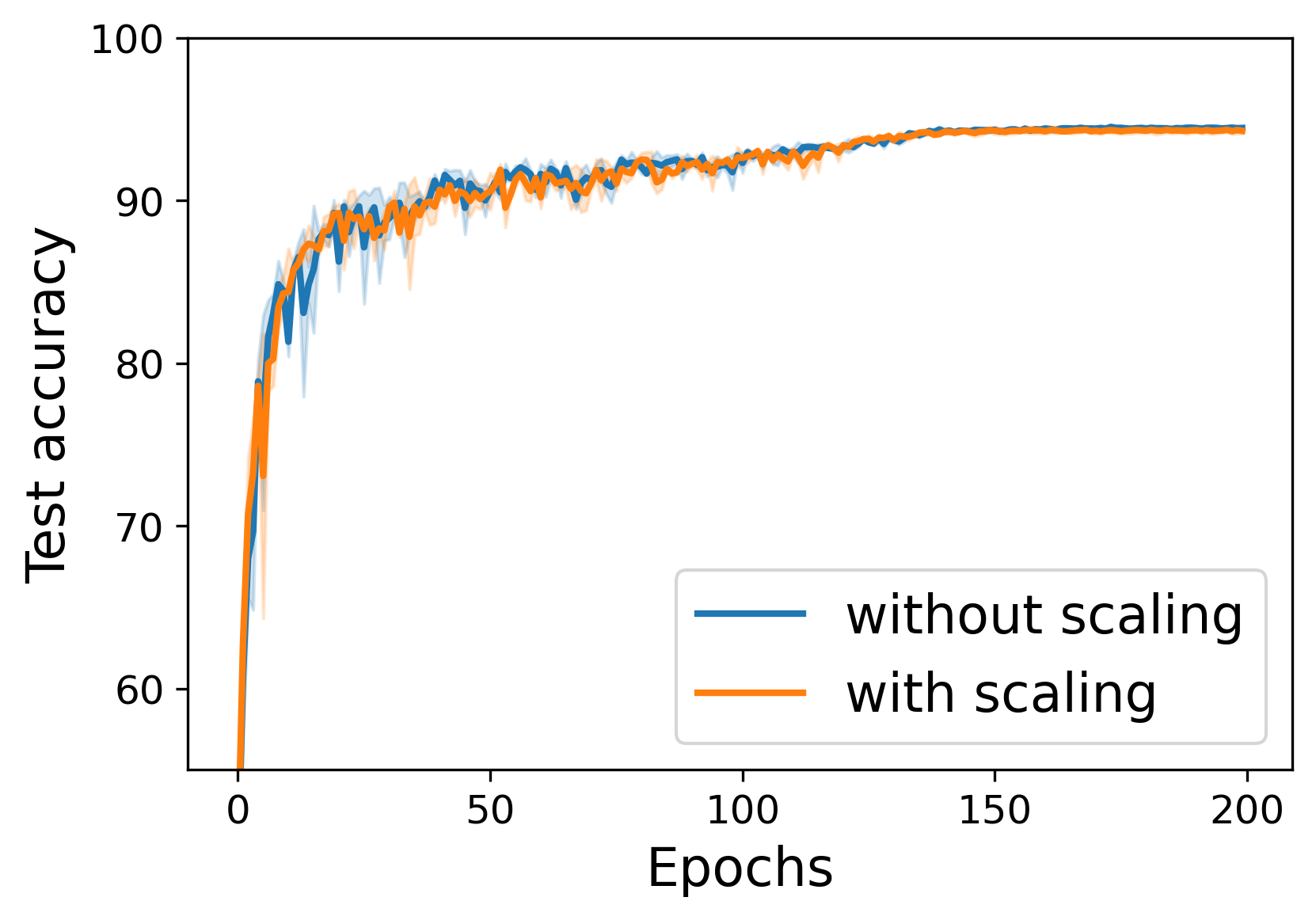}
        \caption{Test Accuracy}
        \label{fig:1c}
    \end{subfigure}
    \caption{Experiments on CIFAR-10 with ResNet-18 Network. The curves represent the average performance of each optimizer in three trials, and the shaded regions denote the standard deviation.}
    \label{fig:experiment}
    \vskip 0.2in
\end{figure*}

\paragraph{Second-order smooth case}
When $F$ is second-order smooth and we set $c=O(1)$, we can check that $\alpha = O(N^{-4/7})$, $\eta=O(N^{-1})$, and $\mu=O(N^{3/7})$ (again $\eta\mu\approx \alpha$). Consequently, the effective momentum should be set to $\tilde\beta \approx 1-\frac{1}{N^{4/7}}$ and the effective learning rate should be $\tilde\eta \approx \frac{1}{N^{3/7}}$. It is interesting to note that in both smooth and second-order smooth cases, $(1-\tilde\beta)\tilde\eta \approx \frac{1}{N}$.
\section{Lower Bounds for finding $(c,\epsilon)$-stationary points}
\label{sec:lower}

% \begin{remark}
In this section we leverage Lemma~\ref{lem:criterion-reduction} to build a lower bound for finding  $(c,\epsilon)$-stationary points. Inuitively, Lemma~\ref{lem:criterion-reduction} suggests that $O(c^{1/2}\epsilon^{-7/2})$ is the optimal rate for finding $(c,\epsilon)$-stationary point. We can indeed prove its optimality using the lower bound construction by \citet{arjevani2019lower} and \citet{cutkosky2023optimal}.

Specifically, \citet{arjevani2019lower} proved the following result: For any constants $H,F^*,\sigma,\epsilon$, there exists objective $F$ and stochastic gradient estimator $\nabla f$ such that (i) $F$ is $H$-smooth, $F(\x_0)-\inf F(\x) \le F^*$, and $\E\|\nabla F(\x) - \nabla f(\x,z)\|^2 \le\sigma^2$; and (ii) any randomized algorithm using $\nabla f$ requires $O(F^*\sigma^2H\epsilon^{-4})$ iterations to find an $\epsilon$-stationary point of $F$.
As a caveat, such construction does not ensure that $F$ is Lipschitz. Fortunately, \citet{cutkosky2023optimal} extended the lower bound construction so that the same lower holds and $F$ is in addition $\sqrt{F^*H}$-Lipschitz.

Consequently, for any $F^*,G,c,\epsilon$, define $H=\sqrt{c\epsilon}$ and $\sigma=G$ and assume $\sqrt{F^*H}\le G$. Then by the lower bound construction, there exists $F$ and $\calO$ such that $F$ is $H$-smooth, $G$-Lipschitz, $F(\x_0)-\inf F(\x) \le F^*$, and $\E\|\nabla F(\x)-\nabla f(\x,z)\|^2 \le G^2$. Lipschitzness and variance bound together imply $\E\|\nabla f(\x,z)\|^2 \le 2G^2$. Moreover, finding an $\epsilon$-stationary of $F$ requires $\Omega(F^*\sigma^2H\epsilon^{-4}) = \Omega(F^*G^2c^{1/2}\epsilon^{-7/2})$ iterations (since $\sigma=G$, $H=\sqrt{c\epsilon}$). 

Finally, note that $H=\sqrt{c\epsilon}$ satisfies $c=H^2\epsilon^{-1}$. Therefore by Lemma \ref{lem:criterion-reduction}, a $(c, \epsilon)$-stationary point of $F$ is also an $\epsilon$-stationary of $F$, implying that finding $(c,\epsilon)$-stationary requires at least $\Omega(F^*G^2c^{1/2}\epsilon^{-7/2})$ iterations as well. Putting these together, we have the following result:
% \end{remark}

\begin{corollary}
\label{cor:lower-bound}
For any $F^*, c, \epsilon$ and $G\ge \sqrt{F^*}(c\epsilon)^{1/4}$, there exists objective $F$ and stochastic gradient $\nabla f$ such that (i) $F$ is $G$-Lipschitz, $F(\x_0)-\inf F(\x) \le F^*$, and $\E\|\nabla f(\x,z)\|^2 \le 2G^2$; and (ii) any randomized algorithm using $\nabla f$ requires $\Omega(F^*G^2c^{1/2}\epsilon^{-7/2})$ iterations to find $(c,\epsilon)$-stationary point of $F$.
\end{corollary}    
\section{Experiments}
\label{sec:experiment}

In the preceding sections, we theoretically demonstrated that scaling the learning rate by an exponential random variable $s_n$ allows SGDM to satisfy convergence guarantees for non-smooth non-convex optimization. To validate this finding empirically, we implemented the SGDM algorithm with random scaling and assessed its performance against the standard SGDM optimizer without random scaling. Our evaluation involved the ResNet-18 model \cite{he2016deep} on the CIFAR-10 image classification benchmark \cite{krizhevsky2009learning}. For the hyperparameters, we configured the learning rate at $0.01$, the momentum constant at $0.9$, and the weight decay at $5 \times 10^{-4}$. These settings are optimized for training the ResNet model on the CIFAR-10 dataset using SGDM. We use the same hyperparameters for our modified SGDM with random scaling. 

For each optimizer, we ran the experiment three times under the same setting to minimize variability. We recorded the train loss, train accuracy, test loss, and test accuracy (refer to Figure \ref{fig:experiment}). We also recorded the performance of the best iterate, e.g., the lowest train/test loss and the highest train/test accuracy, in each trial (see Table \ref{tab:1}).

\begin{table}[h]
\caption{Performance of the best iterate in each trial.}
\label{tab:1}
\vskip 0.15in
\begin{center}
\begin{small}
\begin{sc}
\begin{tabular}{lcccr}
\toprule
Random Scaling & No & Yes  \\
\midrule
Train loss ($\times10^{-4}$)     & 9.82 $\pm$ 0.21  & 9.55 $\pm$ 0.37 \\
Train accuracy ($\%$)            & 100.0 $\pm$ 0.0  & 100.0 $\pm$ 0.0 \\
Test loss ($\times10^{-2})$      & 21.6 $\pm$ 0.1   & 22.0 $\pm$ 0.4 \\
Test accuracy ($\%$)             & 94.6 $\pm$ 0.1   & 94.4 $\pm$ 0.2 \\
\bottomrule
\end{tabular}
\end{sc}
\end{small}
\end{center}
\vskip -0.1in
\end{table}

These results show that the performance of SGDM with random scaling aligns closely with that of standard SGDM. 
% Notably, adding random scaling to SGDM improves its performance on training loss over standard SGDM without random scaling.
\section{Conclusion}

We introduced $(c,\epsilon)$-stationary point, a relaxed definition of Goldstein stationary point, as a new notion of convergence criterion in non-smooth non-convex stochastic optimization. Furthermore, we proposed Exponentiated O2NC, a modified online-to-non-convex framework, by setting exponential random variable as scaling factor and adopting exponentiated and regularized loss. When applied with unconstrained online gradient descent, this framework produces an algorithm that recovers standard SGDM with random scaling and finds $(c,\epsilon)$-stationary point within $O(c^{1/2}\epsilon^{-7/2})$ iterations. Notably, the algorithm automatically achieves the optimal rate of $O(\epsilon^{-4})$ for smooth objectives and $O(\epsilon^{-7/2})$ for second-order smooth objectives.

% Our results raise several interesting questions for further investigation. First, setting $c=\epsilon/\delta^2$ translates the $O(c^{1/2}\epsilon^{-7/2})$ convergence rate of our algorithm to $O(\delta^{-1}\epsilon^{-3})$, the optimal rate for finding $(\delta,\epsilon)$-Goldstein stationary point. However, whether $O(c^{1/2}\epsilon^{-7/2})$ is optimal for finding $(c,\epsilon)$-stationary point remain unclear. We conjecture that this rate is indeed optimal: an improved rate would automatically reduces to a rate better than the optimal $O(\epsilon^{-4})$ for smooth objectives.

One interesting open problem is designing an adaptive algorithm with our Exponentiated O2NC framework. Since our framework, when applied with the simplest OCO algorithm online gradient descent, yields SGDM, a natural question emerges: what if we replace online gradient descent with an adaptive online learning algorithm, such as AdaGrad? Ideally, applied with AdaGrad as the OCO subroutine and with proper tuning, Exponentiated O2NC could recover Adam's update mechanism.
% $\m_{n+1} = (1-\beta_1)\m_n + \beta_1\g_n$, $v_{n+1} = \beta_2v_n + (1-\beta_2)\g_n^2$, and $\x_{n+1} = \x_n - \eta\frac{\m_t}{\sqrt{v_t}}$. 
However, the convergence analysis for this scenario is complex and demands a nuanced approach, especially considering the intricacies associated with the adaptive learning rate.
In this vein, concurrent work by \citet{ahn2024understanding} applies a similar concept of online-to-non-convex conversion and connects the Adam algorithm to a principled online learning family known as Follow-The-Regularized-Leader (FTRL).
\section*{Acknowledgment}

We thank Dingzhi Yu for pointing out that original version of Lemma \ref{lem:goldstein-reduction} held only for the classical version of Goldstein stationarity.

\bibliography{references}

@inproceedings{zinkevich2003online,
  title={Online Convex Programming and Generalized Infinitesimal Gradient Ascent},
  author={Zinkevich, Martin},
  booktitle={Proceedings of the 20th International Conference on Machine Learning (ICML-03)},
  pages={928--936},
  year={2003}
}

@article{orabona2016scale,
  title={Scale-Free Online Learning},
  author={Orabona, Francesco and P{\'a}l, D{\'a}vid},
  journal={arXiv preprint arXiv:1601.01974},
  year={2016}
}

@article{kingma2014adam,
  title={Adam: A method for stochastic optimization},
  author={Kingma, Diederik and Ba, Jimmy},
  journal={arXiv preprint arXiv:1412.6980},
  year={2014}
}

@misc{krizhevsky2009learning,
  title={Learning multiple layers of features from tiny images},
  author={Krizhevsky, Alex and Hinton, Geoffrey},
  year={2009},
  publisher={Citeseer}
}

@inproceedings{cutkosky2019momentum,
  title={Momentum-Based Variance Reduction in Non-convex SGD},
  author={Cutkosky, Ashok and Orabona, Francesco},
  booktitle={Advances in Neural Information Processing Systems},
  year={2019},
  pages={15210--15219}
}

@article{arjevani2019lower,
  title={Lower bounds for non-convex stochastic optimization},
  author={Arjevani, Yossi and Carmon, Yair and Duchi, John C and Foster, Dylan J and Srebro, Nathan and Woodworth, Blake},
  journal={arXiv preprint arXiv:1912.02365},
  year={2019}
}

@inproceedings{arjevani2020second,
  title={Second-Order Information in Non-Convex Stochastic Optimization: Power and Limitations},
  author={Arjevani, Yossi and Carmon, Yair and Duchi, John C and Foster, Dylan J and Sekhari, Ayush and Sridharan, Karthik},
  booktitle={Conference on Learning Theory},
  pages={242--299},
  year={2020}
}

@inproceedings{fang2018spider,
  title={Spider: Near-optimal non-convex optimization via stochastic path-integrated differential estimator},
  author={Fang, Cong and Li, Chris Junchi and Lin, Zhouchen and Zhang, Tong},
  booktitle={Advances in Neural Information Processing Systems},
  pages={689--699},
  year={2018}
}

@inproceedings{fang2019sharp,
  title={Sharp Analysis for Nonconvex SGD Escaping from Saddle Points},
  author={Fang, Cong and Lin, Zhouchen and Zhang, Tong},
  booktitle={Conference on Learning Theory},
  pages={1192--1234},
  year={2019}
}

@inproceedings{cutkosky2020momentum,
  title={Momentum Improves Normalized SGD},
  author={Cutkosky, Ashok and Mehta, Harsh},
  booktitle={International Conference on Machine Learning},
  year={2020}
}

@inproceedings{allen2018natasha,
  title={Natasha 2: Faster non-convex optimization than sgd},
  author={Allen-Zhu, Zeyuan},
  booktitle={Advances in neural information processing systems},
  pages={2675--2686},
  year={2018}
}

@inproceedings{tripuraneni2018stochastic,
  title={Stochastic cubic regularization for fast nonconvex optimization},
  author={Tripuraneni, Nilesh and Stern, Mitchell and Jin, Chi and Regier, Jeffrey and Jordan, Michael I},
  booktitle={Advances in neural information processing systems},
  pages={2899--2908},
  year={2018}
}

@article{ghadimi2013stochastic,
  title={Stochastic first-and zeroth-order methods for nonconvex stochastic programming},
  author={Ghadimi, Saeed and Lan, Guanghui},
  journal={SIAM Journal on Optimization},
  volume={23},
  number={4},
  pages={2341--2368},
  year={2013},
  publisher={SIAM}
}

@inproceedings{zhou2018stochastic,
  title={Stochastic nested variance reduction for nonconvex optimization},
  author={Zhou, Dongruo and Xu, Pan and Gu, Quanquan},
  booktitle={Proceedings of the 32nd International Conference on Neural Information Processing Systems},
  pages={3925--3936},
  year={2018}
}

@article{hazan2019introduction,
  title={Introduction to online convex optimization},
  author={Hazan, Elad},
  journal={arXiv preprint arXiv:1909.05207},
  year={2019}
}

@article{orabona2019modern,
  title={A modern introduction to online learning},
  author={Orabona, Francesco},
  journal={arXiv preprint arXiv:1912.13213},
  year={2019}
}

@inproceedings{he2016deep,
  title={Deep residual learning for image recognition},
  author={He, Kaiming and Zhang, Xiangyu and Ren, Shaoqing and Sun, Jian},
  booktitle={Proceedings of the IEEE conference on computer vision and pattern recognition},
  pages={770--778},
  year={2016}
}

@article{you2017scaling,
  title={Scaling sgd batch size to 32k for imagenet training},
  author={You, Yang and Gitman, Igor and Ginsburg, Boris},
  journal={arXiv preprint arXiv:1708.03888},
  volume={6},
  pages={12},
  year={2017}
}

@inproceedings{carmon2017convex,
  title={“Convex Until Proven Guilty”: Dimension-Free Acceleration of Gradient Descent on Non-Convex Functions},
  author={Carmon, Yair and Duchi, John C and Hinder, Oliver and Sidford, Aaron},
  booktitle={International Conference on Machine Learning},
  pages={654--663},
  year={2017},
  organization={PMLR}
}

@article{arjevani2022lower,
  title={Lower bounds for non-convex stochastic optimization},
  author={Arjevani, Yossi and Carmon, Yair and Duchi, John C and Foster, Dylan J and Srebro, Nathan and Woodworth, Blake},
  journal={Mathematical Programming},
  pages={1--50},
  year={2022},
  publisher={Springer}
}

@article{carmon2019lower,
  title={Lower bounds for finding stationary points i},
  author={Carmon, Yair and Duchi, John C and Hinder, Oliver and Sidford, Aaron},
  journal={Mathematical Programming},
  pages={1--50},
  year={2019},
  publisher={Springer}
}

@article{you2019large,
  title={Large batch optimization for deep learning: Training bert in 76 minutes},
  author={You, Yang and Li, Jing and Reddi, Sashank and Hseu, Jonathan and Kumar, Sanjiv and Bhojanapalli, Srinadh and Song, Xiaodan and Demmel, James and Keutzer, Kurt and Hsieh, Cho-Jui},
  journal={arXiv preprint arXiv:1904.00962},
  year={2019}
}

@article{article,
author = {He, Jianxing and Baxter, Sally and Xu, Jie and Xu, Jiming and Zhou, Xingtao and Zhang, Kang},
year = {2019},
month = {01},
pages = {},
title = {The practical implementation of artificial intelligence technologies in medicine},
volume = {25},
journal = {Nature Medicine},
doi = {10.1038/s41591-018-0307-0}
}

@inproceedings{duchi2010composite,
  title={Composite Objective Mirror Descent.},
  author={Duchi, John C and Shalev-Shwartz, Shai and Singer, Yoram and Tewari, Ambuj},
  booktitle={COLT},
  volume={10},
  pages={14--26},
  year={2010},
  organization={Citeseer}
}

@InProceedings{jacobsen2022parameter,
  title = 	 {Parameter-free Mirror Descent},
  author =       {Jacobsen, Andrew and Cutkosky, Ashok},
  booktitle = 	 {Proceedings of Thirty Fifth Conference on Learning Theory},
  pages = 	 {4160--4211},
  year = 	 {2022},
  editor = 	 {Loh, Po-Ling and Raginsky, Maxim},
  volume = 	 {178},
  series = 	 {Proceedings of Machine Learning Research},
  publisher =    {PMLR},
  url = 	 {https://proceedings.mlr.press/v178/jacobsen22a.html}
}

@article{zhang2020complexity,
  title={Complexity of finding stationary points of nonsmooth nonconvex functions},
  author={Zhang, Jingzhao and Lin, Hongzhou and Jegelka, Stefanie and Jadbabaie, Ali and Sra, Suvrit},
  booktitle={International Conference on Machine Learning},
  year={2020}
}

@article{kornowski2022oracle,
  title={Oracle Complexity in Nonsmooth Nonconvex Optimization},
  author={Kornowski, Guy and Shamir, Ohad},
  journal={Journal of Machine Learning Research},
  volume={23},
  number={314},
  pages={1--44},
  year={2022}
}

@article{kornowski2022complexity,
  title={On the Complexity of Finding Small Subgradients in Nonsmooth Optimization},
  author={Kornowski, Guy and Shamir, Ohad},
  journal={arXiv preprint arXiv:2209.10346},
  year={2022}
}

@misc{kornowski2023algorithm,
      title={An Algorithm with Optimal Dimension-Dependence for Zero-Order Nonsmooth Nonconvex Stochastic Optimization}, 
      author={Guy Kornowski and Ohad Shamir},
      year={2023},
      eprint={2307.04504},
      archivePrefix={arXiv},
      primaryClass={math.OC}
}

@misc{lin2022gradientfree,
      title={Gradient-Free Methods for Deterministic and Stochastic Nonsmooth Nonconvex Optimization}, 
      author={Tianyi Lin and Zeyu Zheng and Michael I. Jordan},
      year={2022},
      eprint={2209.05045},
      archivePrefix={arXiv},
      primaryClass={math.OC}
}

@article{carmon2018accelerated,
  title={Accelerated methods for nonconvex optimization},
  author={Carmon, Yair and Duchi, John C and Hinder, Oliver and Sidford, Aaron},
  journal={SIAM Journal on Optimization},
  volume={28},
  number={2},
  pages={1751--1772},
  year={2018},
  publisher={SIAM}
}

@inproceedings{cutkosky2023optimal,
  title={Optimal Stochastic Non-smooth Non-convex Optimization through Online-to-Non-convex Conversion},
  author={Cutkosky, Ashok and Mehta, Harsh and Orabona, Francesco},
  booktitle = {International Conference on Machine Learning (ICML)},
  year={2023}
}

@book{cesa2006prediction,
  title={Prediction, learning, and games},
  author={Cesa-Bianchi, Nicolo and Lugosi, G{\'a}bor},
  year={2006},
  publisher={Cambridge university press}
}

@article{goldstein1977optimization,
  author = {Goldstein, A. A.},
  title = {Optimization of Lipschitz Continuous Functions},
  year = {1977},
  issue_date = {December  1977},
  publisher = {Springer-Verlag},
  address = {Berlin, Heidelberg},
  volume = {13},
  number = {1},
  issn = {0025-5610},
  url = {https://doi.org/10.1007/BF01584320},
  doi = {10.1007/BF01584320},
  journal = {Math. Program.},
  month = {dec},
  pages = {14–22},
  numpages = {9},
}

@InProceedings{sutskever2013importance,
  title = 	 {On the importance of initialization and momentum in deep learning},
  author = 	 {Sutskever, Ilya and Martens, James and Dahl, George and Hinton, Geoffrey},
  booktitle = 	 {Proceedings of the 30th International Conference on Machine Learning},
  pages = 	 {1139--1147},
  year = 	 {2013},
  volume = 	 {28},
  number =       {3},
}

@article{chen1993convergence,
    author = {Chen, Gong and Teboulle, Marc},
    title = {Convergence Analysis of a Proximal-Like Minimization Algorithm Using Bregman Functions},
    journal = {SIAM Journal on Optimization},
    volume = {3},
    number = {3},
    pages = {538-543},
    year = {1993},
}

@misc{liu2021laprop,
      title={LaProp: Separating Momentum and Adaptivity in Adam}, 
      author={Liu Ziyin and Zhikang T. Wang and Masahito Ueda},
      year={2021},
      eprint={2002.04839},
      archivePrefix={arXiv},
      primaryClass={cs.LG}
}

@misc{jordan2023deterministic,
      title={Deterministic Nonsmooth Nonconvex Optimization}, 
      author={Michael I. Jordan and Guy Kornowski and Tianyi Lin and Ohad Shamir and Manolis Zampetakis},
      year={2023},
      eprint={2302.08300},
      archivePrefix={arXiv},
      primaryClass={cs.LG}
}

@misc{fang2021online,
      title={Online mirror descent and dual averaging: keeping pace in the dynamic case}, 
      author={Huang Fang and Nicholas J. A. Harvey and Victor S. Portella and Michael P. Friedlander},
      year={2021},
      eprint={2006.02585},
      archivePrefix={arXiv},
      primaryClass={cs.LG}
}

@article{beck2003mirror,
  title={Mirror descent and nonlinear projected subgradient methods for convex optimization},
  author={Beck, Amir and Teboulle, Marc},
  journal={Operations Research Letters},
  volume={31},
  number={3},
  pages={167--175},
  year={2003},
  publisher={Elsevier}
}

@InProceedings{mai2020convergence,
  title = 	 {Convergence of a Stochastic Gradient Method with Momentum for Non-Smooth Non-Convex Optimization},
  author =       {Mai, Vien and Johansson, Mikael},
  booktitle = 	 {Proceedings of the 37th International Conference on Machine Learning},
  pages = 	 {6630--6639},
  year = 	 {2020},
  editor = 	 {III, Hal Daumé and Singh, Aarti},
  volume = 	 {119},
  series = 	 {Proceedings of Machine Learning Research},
  month = 	 {13--18 Jul},
  publisher =    {PMLR},
}

@article{davis2019stochastic,
    author = {Davis, Damek and Drusvyatskiy, Dmitriy},
    title = {Stochastic Model-Based Minimization of Weakly Convex Functions},
    journal = {SIAM Journal on Optimization},
    volume = {29},
    number = {1},
    pages = {207-239},
    year = {2019},
    doi = {10.1137/18M1178244},
}

@misc{ahn2024understanding,
      title={Understanding Adam Optimizer via Online Learning of Updates: Adam is FTRL in Disguise}, 
      author={Kwangjun Ahn and Zhiyu Zhang and Yunbum Kook and Yan Dai},
      year={2024},
      eprint={2402.01567},
      archivePrefix={arXiv},
      primaryClass={cs.LG}
}
\bibliographystyle{apalike}

%%%%%%%%%%%%%%%%%%%%%%%%%%%%%%%%%%%%%%%%%%%%%%%%%%%%%%%%%%%%%%%%%%%%%%%%%%%%%%%
%%%%%%%%%%%%%%%%%%%%%%%%%%%%%%%%%%%%%%%%%%%%%%%%%%%%%%%%%%%%%%%%%%%%%%%%%%%%%%%
% APPENDIX
%%%%%%%%%%%%%%%%%%%%%%%%%%%%%%%%%%%%%%%%%%%%%%%%%%%%%%%%%%%%%%%%%%%%%%%%%%%%%%%
%%%%%%%%%%%%%%%%%%%%%%%%%%%%%%%%%%%%%%%%%%%%%%%%%%%%%%%%%%%%%%%%%%%%%%%%%%%%%%%
\newpage
\appendix
% \onecolumn

\section{Proofs in Section \ref{sec:pre}}
\label{app:pre}

\subsection{Proof of Lemma \ref{lem:criterion-reduction}}

\CriterionReduction*

\begin{proof}
Suppose $\|\nabla F(\x)\|_c \le \epsilon$, then there exists $P\in\calP(S), y\sim P$ such that $\E[\y]=\x$, $\|\E \nabla F(\y)\| \le \epsilon$ and $c\E\|\y-\x\|^2 \le \epsilon$. 

Assume $F$ is $H$-smooth. By Jensen's inequality, $\E\|\y-\x\| \le \sqrt{\epsilon/c} = \epsilon/H$ with $c=H^2\epsilon^{-1}$. Consequently,
\begin{align*}
    \|\nabla F(\x)\|
    &\le \|\E\nabla F(\y)\| + \|\E[\nabla F(\x) - \nabla F(\y)]\| \\
    &\le \|\E\nabla F(\y)\| + \E\|\nabla F(\x) - \nabla F(\y)\| \tag{Jensen's inequality}\\
    &\le \|\E\nabla F(\y)\| + H\E\|\x-\y\| \tag{smoothness}\\
    &\le \epsilon + H \cdot \epsilon/H = 2\epsilon. 
    % \tag{$\delta=\frac{\epsilon}{H}$}
\end{align*}
Next, assume $F$ is $\rho$-second-order smooth. By Taylor approximation, there exists some $\z$ such that $\nabla F(\x) = \nabla F(\y) + \nabla^2F(\x)(\x-\y) + \frac{1}{2}(\x-\y)^T\nabla^3F(\z)(\x-\y)$. Note that $\E[\nabla^2F(\x)(\x-\y)]=\nabla^2F(\x)\E[\x-\y]=0$. Consequently,
\begin{align*}
    \|\nabla F(\x)\|
    &\le \|\E\nabla F(\y)\| + \|\E[\nabla F(\x)-\nabla F(\y)]\| \\
    &\le \|\E\nabla F(\y)\| + \E\|\tfrac{1}{2}(\x-\y)^T\nabla^3F(\z)(\x-\y)\| \tag{Jensen's inequality}\\
    &\le \|\E\nabla F(\y)\| + \tfrac{\rho}{2} \E\|\x-\y\|^2 \tag{second-order-smooth}\\
    &\le \epsilon + \tfrac{\rho}{2}\cdot \epsilon/c = 2\epsilon. \tag{$c=\rho/2$}
\end{align*}
Together these prove the reduction from a $(c,\epsilon)$-stationary point to an $\epsilon$-stationary point.
\end{proof}

\subsection{Proof of Lemma \ref{lem:goldstein-reduction}}

\GoldsteinReduction*

\begin{proof}
By definition of $(c,\epsilon)$-stationary, there exists some distribution of $\y$ such that $\E[\y]=\x$, $\sigma^2:=\E\|\y-\x\|^2 \le \epsilon/c$, and $\|\E \nabla F(\y)\| \le \epsilon$. By Chebyshev’s inequality, 
\begin{align*}
    \P\{ \|\y-\x\| \ge \delta \}
    &= \P\left\{ \|\y-\E[\y]\| \ge \frac{\delta}{\sigma} \cdot\sigma \right\} \\
    &\le \P\left\{ \|\y-\E[\y]\| \ge \frac{\delta}{\sqrt{\epsilon/c}} \cdot\sigma \right\} 
    \le \frac{\epsilon}{c\delta^2}.
\end{align*}
Next, we can construct a clipped random vector $\hat \y\coloneqq \x + \min\{1,\frac{\delta}{\|\y-\x\|} \}(\y-\x)$, and let $\x'\coloneqq\E[\hat \y]$. Then 
\begin{align*}
    \|\x-\x'\| 
    &= \left\| \E\left[ \min\{1,\frac{\delta}{\|\y-\x\|} \}(\y-\x) \right] \right\| \\
    &\le \E \left\| \min\{1,\frac{\delta}{\|\y-\x\|} \}(\y-\x) \right\| \le\delta. \tag{Jensen's inequality}
\end{align*}
In particular, $\|\hat\y-\x'\| \le \|\hat\y-\x\| + \|\x-\x'\| \le 2\delta$.
Moreover, note that $\P\{\hat \y\ne \y\} = \P\{\|\y-\x\|\ge \delta\} \le \frac{\epsilon}{c\delta^2}$. Since $F$ is $G$-Lipschitz,
\begin{align*}
    \|\E[\nabla F(\hat\y) - \nabla F(\y)]\| 
    &= \P\{\hat \y \ne \y\} \|\E[\nabla F(\hat \y)-\nabla F(\y) | \hat\y\ne \y]\| \\
    &\le 2G\cdot \P\{\hat\y\ne \y\}
    \le 2G \cdot \frac{\epsilon}{c\delta^2}.
\end{align*}
Consequently $\|\E[\nabla F(\hat \y)]\| \le \|\E[\nabla F(\y)]\| + \|\E[\nabla F(\hat \y) - \nabla F(\y)]\| \le \epsilon +\frac{2G\epsilon}{c\delta^2}$. This proves that $\x'$ is a $(2\delta, \epsilon + \frac{2G\epsilon}{c\delta^2})$-Goldstein stationary point. 
\end{proof}
\section{Proofs in Section \ref{sec:alg}}

\subsection{Proof of Lemma \ref{lem:variance-regularizer}}

The proof consists of two composite lemmas.
Recall the following notations: $S_n = \{\x_t\}_{t\in[n]}$, $\y_n\sim P_n$ where $P_n(\x_t) = \beta^{n-t}\cdot \frac{1-\beta}{1-\beta^n}$, and $\overline \x_n = \sum_{t=1}^n \beta^{n-t}\x_t\cdot \frac{1-\beta}{1-\beta^n}$. Also note two useful change of summation identities:
\begin{align*}
    &\sum_{n=1}^N\sum_{t=1}^n = \sum_{1\le t\le n\le N} = \sum_{t=1}^N\sum_{n=t}^N, 
    &\sum_{i=1}^n\sum_{i'=1}^{i-1}\sum_{t=i'+1}^i = \sum_{1\le i'<t\le i\le n} = \sum_{t=1}^n\sum_{i=t}^n\sum_{i'=1}^{t-1}.
\end{align*}

\begin{proposition}
\label{prop:variance-regularizer}
$\E_{\y_n,s}\|\y_n-\overline\x_n\|^2 \le \sum_{t=1}^n \lambda_{n,t}\|\Delta_t\|^2$, where
\begin{align}
    \lambda_{n,t} = 4\sum_{i=t}^n\sum_{i'=1}^{t-1} p_{n,i}p_{n,i'}(i-i'), 
    \quad p_{n,i} = P_n(\x_i) = \beta^{n-i}\cdot\frac{1-\beta}{1-\beta^n}. 
    \label{eq:def-lambda}
\end{align}
\end{proposition}

\begin{proof}
By distribution of $\y_n$, we have
\begin{align*}
    \Ex_{\y_n} \|\y_n - \overline\x_n\|^2 
    &= \sum_{i=1}^n p_{n,i} \|\x_i - \overline \x_n\|^2 \\
    &= \sum_{i=1}^n p_{n,i} \left\| \sum_{i'=1}^n p_{n,i'}(\x_i-\x_{i'}) \right\|^2 \\
    &\le \sum_{i=1}^n\sum_{i'=1}^n p_{n,i}p_{n,i'} \|\x_i-\x_{i'}\|^2 
    = 2\sum_{i=1}^n\sum_{i'=1}^{i-1} p_{n,i}p_{n,i'}\|\x_i-\x_{i'}\|^2.
\end{align*}
The inequality uses convexity of $\|\cdot\|^2$. Next, upon unrolling the recursive update $\x_t = \x_{t-1} + \s_t\Delta_t$,
\begin{align*}
    \|\x_i-\x_{i'}\|^2
    &= \left\| \sum_{t=i'+1}^i s_t\Delta_t\right\|^2 
    \le (i-i') \sum_{t=i'+1}^i s_t^2\|\Delta_t\|^2.
\end{align*}
Note that $s_t$ and $\Delta_t$ are independent and $s_t\sim \Expo(1)$, so $\E_s[s_t^2\|\Delta_t\|^2] = \E_s[s_t^2] \|\Delta_t\|^2 = 2\|\Delta_t\|^p$. Consequently, upon substituting this back and applying change of summation, we have
\begin{align*}
    \Ex_{\y_n,s} \|\y_n-\overline\x_n\|^2 
    &\le 4\sum_{i=1}^n\sum_{i'=1}^{i-1}\sum_{t=i'+1}^i p_{n,i}p_{n,i'}(i-i') \|\Delta_t\|^2 \\
    &= \sum_{t=1}^n \left(4\sum_{i=t}^n\sum_{i'=1}^{t-1}p_{n,i}p_{n,i'}(i-i')\right) \|\Delta_t\|^2.
\end{align*}
We then conclude the proof by substituting the definition of $\lambda_{n,t}$.
\end{proof}

\begin{proposition}
\label{prop:tech-lam-t:N}
Define $\lambda_{n,t}$ as in \eqref{eq:def-lambda}, then $\sum_{n=t}^N \lambda_{n,t} \le \frac{12}{(1-\beta)^2}$.
% \begin{align*}
%     % &\lambda_{n,t} = 4\cdot \frac{(n-t+1)\beta^{n-t+1} - n\beta^n + (t-1)\beta^{2n-t+1}}{(1-\beta^n)^2}, 
%     &\sum_{n=t}^N \lambda_{n,t} \le \frac{12}{(1-\beta)^2}.
% \end{align*}
\end{proposition}

\begin{proof}
In the first part of the proof, we find a good upper bound of $\lambda_{n,t}$. We can rearrange the definition of $\lambda_{n,t}$ as follows.
\begin{align}
    \lambda_{n,t}
    &= 4\left(\frac{1-\beta}{1-\beta^n}\right)^2 \sum_{i=t}^n\sum_{i'=1}^{t-1} \beta^{n-i}\beta^{n-i'}(i-i') \tag{let $j=i-i'$}\\
    &= 4\left(\frac{1-\beta}{1-\beta^n}\right)^2 \sum_{i=t}^n\sum_{j=i-t+1}^{i-1} \beta^{n-i}\beta^{n-i+j}\cdot j \tag{let $k=n-i$}\\
    &= 4\left(\frac{1-\beta}{1-\beta^n}\right)^2 \sum_{k=0}^{n-t} \beta^{2k} \sum_{j=n-k-t+1}^{n-k-1} j\beta^j.
    \label{eq:lam-sum-2a}
\end{align}
The second line uses change of variable that $j=i-i'$, and the third line uses $k=n-i$. Next, 
\begin{align*}
    \sum_{j=n-k-t+1}^{n-k-1} j\beta^j
    = \beta\sum_{j=n-k-t+1}^{n-k-1} \frac{d}{d\beta} \beta^j 
    &= \beta\cdot \frac{d}{d\beta} \left( \sum_{j=n-k-t+1}^{n-k-1} \beta^j \right) \\
    &= \beta\cdot \frac{d}{d\beta} \left( \frac{\beta^{n-k-t+1} - \beta^{n-k}}{1-\beta} \right) \\
    &= \frac{\beta^{a-k+1}-\beta^{b-k+1}}{(1-\beta)^2} + \frac{(a-k)\beta^{a-k} - (b-k)\beta^{b-k}}{1-\beta},
\end{align*}
where $a=n-t+1, b=n$. Upon substituting this back into \eqref{eq:lam-sum-2a}, we have
\begin{align}
    \lambda_{n,t}
    &= 4\left(\frac{1-\beta}{1-\beta^n}\right)^2 \sum_{k=0}^{n-t} \beta^{2k}\left( \frac{\beta^{a-k+1} - \beta^{b-k+1}}{(1-\beta)^2} + \frac{a\beta^{a-k}-b\beta^{b-k}}{1-\beta} - k \frac{\beta^{a-k}-\beta^{b-k}}{1-\beta} \right) \notag\\
    &= 4\left(\frac{1-\beta}{1-\beta^n}\right)^2 \sum_{k=0}^{n-t} \left( \frac{\beta^{a+1} - \beta^{b+1}}{(1-\beta)^2} + \frac{a\beta^{a}-b\beta^{b}}{1-\beta} \right) \beta^k - \frac{\beta^a-\beta^b}{1-\beta} \cdot k\beta^k.
    \label{eq:lam-sum-2b}
\end{align}
For the first term, $\sum_{k=0}^{n-t}\beta^k = \frac{1-\beta^{n-t+1}}{1-\beta} = \frac{1-\beta^a}{1-\beta}$. For the second term,
\begin{align*}
    \sum_{k=0}^{n-t} k\beta^k
    = \beta \cdot \frac{d}{d\beta} \left( \sum_{k=0}^{n-t}\beta^k \right) 
    = \beta \cdot \frac{d}{d\beta} \left( \frac{1-\beta^a}{1-\beta} \right) 
    = \frac{\beta-\beta^{a+1}}{(1-\beta)^2} - \frac{a\beta^a}{1-\beta}.
\end{align*}
Upon substituting this back into \eqref{eq:lam-sum-2b} and simplifying the expression, we have
\begin{align*}
    \lambda_{n,t}
    &= 4\left(\frac{1-\beta}{1-\beta^n}\right)^2 \cdot \left[
    \left( \frac{\beta^{a+1} - \beta^{b+1}}{(1-\beta)^2} + \frac{a\beta^{a}-b\beta^{b}}{1-\beta} \right) \cdot \frac{1-\beta^a}{1-\beta}
    - \frac{\beta^a-\beta^b}{1-\beta} \cdot \left(\frac{\beta-\beta^{a+1}}{(1-\beta)^2} - \frac{a\beta^a}{1-\beta}\right)
    \right] \\
    &= 4\frac{(a\beta^a-b\beta^b)(1-\beta^a) + a\beta^a(\beta^a-\beta^b)}{(1-\beta^n)^2} 
    % &= \frac{a\beta^a - b\beta^b + (b-a)\beta^{a+b}}{(1-\beta^n)^2} \\
    = \ldots 
    = 4\frac{a\beta^a(1-\beta^b) - b\beta^b(1-\beta^a)}{(1-\beta^n)^2}.
\end{align*}
Upon substituting $a=n-t+1$ and $b=n$, we conclude the first half of the proof with 
% the following simplified expression of $\lambda_{n,t}$:
\begin{align*}
    \lambda_{n,t} 
    % &= 4\cdot \frac{(n-t+1)\beta^{n-t+1}}{1-\beta^n} - 4\cdot\frac{n\beta^n(1-\beta^{n-t+1})}{(1-\beta^n)^2} \\
    \le 4\frac{a\beta^a(1-\beta^b)}{(1-\beta^n)^2}
    \le 4\cdot \frac{(n-t+1)\beta^{n-t+1}}{1-\beta^n}.
\end{align*}
% Note that the second term is positive, so $\lambda_{n,t} \le 4\cdot \frac{(n-t+1)\beta^{n-t+1}}{1-\beta^n}$. However, we keep the second term for a tighter analysis.
In the second part, we use this inequality to bound $\sum_{n=t}^N\lambda_{n,t}$. Define $K=\lceil \frac{1}{1-\beta}\rceil$, then
\begin{align}
    \sum_{n=t}^N \lambda_{n,t}
    % &\le \sum_{n=1}^N \lambda_{n,t} 
    % \le \sum_{n=t}^{K-1} \lambda_{n,t} + \sum_{n=K}^N \lambda_{n,t},
    = \one_{\{t\le K-1\}}\cdot \sum_{n=t}^{K-1} \lambda_{n,t} + \sum_{n=\max\{t,K\}}^N \lambda_{n,t}.
    \label{eq:lam-sum-1}
\end{align}
% where we denote $\sum_{n=t}^{K-1} = 0$ if $t>K-1$ for simplicity.
For the first summation in \eqref{eq:lam-sum-1}, for all $t\le n\le K-1$, we have
\begin{align*}
    \lambda_{n,t}
    &\le 4\cdot \frac{(n-t+1)\beta^{n-t+1}}{1-\beta^n} 
    \stackrel{(i)}{\le} 4\cdot \frac{(n-t+1)\beta^{n-t+1}}{1-\beta^{n-t+1}}
    \stackrel{(ii)}{\le} 4\cdot \frac{1\cdot \beta^1}{1-\beta^1} \le \frac{4}{1-\beta}.
\end{align*}
(i) holds because $\frac{1}{1-\beta^n}$ is decreasing w.r.t. $n$. (ii) holds because $f(x)=\frac{x\beta^x}{1-\beta^x}$ is decreasing for $x\ge 0$ and $\beta\in(0,1)$, so $f(n-t+1) \le f(1)$ since $n-t+1 \ge 1$.
Recall that $K-1\le \frac{1}{1-\beta}$, then the first summation in \eqref{eq:lam-sum-1} can be bounded by
\begin{align}
    \one_{\{t\le K-1\}}\cdot \sum_{n=t}^{K-1} \lambda_{n,t}
    &\le \sum_{n=1}^{K-1} \frac{4}{1-\beta} \le \frac{4}{(1-\beta)^2}.
    \label{eq:lam-sum-1a}
\end{align}
For the second summation in \eqref{eq:lam-sum-1}, for all $n\ge K \ge \frac{1}{1-\beta}$, 
\begin{align*}
    \frac{1}{1-\beta^n} 
    \stackrel{(i)}{\le} \frac{1}{1-\beta^{\frac{1}{1-\beta}}} 
    \stackrel{(ii)}{\le} \lim_{x\to 1} \frac{1}{1-x^{\frac{1}{1-x}}} = \frac{e}{e-1} \le 2.
\end{align*}
(i) holds because $\frac{1}{1-\beta^n}$ is decreasing. (ii) holds because $f(x)=\frac{1}{1-x^{\frac{1}{1-x}}}$ is increasing for $x\ge 0$, so $f(\beta) \le \lim_{x\to 1}f(x)$ for all $\beta\in(0,1)$. Consequently, the second summation in \eqref{eq:lam-sum-1} can be bounded by
\begin{align}
    \sum_{n=\max\{t,K\}}^N \lambda_{n,t}
    \le \sum_{n=\max\{t,K\}}^N 4\cdot \frac{(n-t+1)\beta^{n-t+1}}{1-\beta^n} 
    \le 8\sum_{n=t}^N (n-t+1)\beta^{n-t+1} 
    = 8\sum_{n=1}^{N-t} n\beta^n
    % \le \frac{8}{(1-\beta)^2}.
    \label{eq:lam-sum-1b}
\end{align}
By change of summation,
\begin{align*}
    \sum_{n=1}^N n\beta^n
    &= \sum_{n=1}^N \sum_{i=1}^n \beta^n 
    = \sum_{i=1}^N \sum_{n=i}^N \beta^n
    \le \sum_{i=1}^N \frac{\beta^i}{1-\beta}
    \le \frac{1}{(1-\beta)^2}.
\end{align*}
We then conclude the proof by substituting \eqref{eq:lam-sum-1a}, \eqref{eq:lam-sum-1b} into \eqref{eq:lam-sum-1}.
\end{proof}

\VarianceToRegularizer*

\begin{proof}
By Proposition \ref{prop:variance-regularizer} and Proposition \ref{prop:tech-lam-t:N}, we have
\begin{align*}
    \Ex_s\sum_{n=1}^N\Ex_{\y_n} \|\y_n-\overline\x_n\|^2 
    \stackrel{(i)}{\le} \sum_{n=1}^N\sum_{t=1}^n \lambda_{n,t}\|\Delta_t\|^2 
    \stackrel{(ii)}{=} \sum_{t=1}^N\left(\sum_{n=t}^N \lambda_{n,t}\right) \|\Delta_t\|^2 
    \stackrel{(iii)}{\le} \sum_{t=1}^N \frac{12}{(1-\beta)^2} \|\Delta_t\|^2. 
\end{align*}
Here (i) is from Proposition \ref{prop:variance-regularizer}, (ii) is from change of summation, and (iii) is from Proposition \ref{prop:tech-lam-t:N}.
\end{proof}

\subsection{Proof of Lemma \ref{lem:exp-scaling}}

\ExponentialScaling*
\begin{proof}
Denote $p(s)=\lambda\exp(-\lambda s)$ as the pdf of $s$. Upon expanding the expectation, we can rewrite the LHS as
\begin{align*}
    \Ex_s [F(\x+s\Delta) - F(\x)] 
    &= \int_0^\infty [F(\x+s\Delta) - F(\x)] p(s) \, ds \\
    &\stackrel{(i)}{=} \int_0^\infty \left( \int_0^s \langle \nabla F(\x+t\Delta), \Delta\rangle \, dt \right) p(s)\, ds \\
    &= \int_0^\infty\int_0^\infty \langle \nabla F(\x+t\Delta), \Delta\rangle \one\{t\le s\}p(s) \, dtds \\
    &= \int_0^\infty \left(\int_t^\infty p(s)\, ds\right) \langle \nabla F(\x+t\Delta),\Delta\rangle \, dt \\
    &\stackrel{(ii)}{=} \int_0^\infty \frac{p(t)}{\lambda} \langle \nabla F(\x+t\Delta),\Delta\rangle \, dt \\
    &= \frac{1}{\lambda} \Ex_s [\langle\nabla F(\x+s\Delta), \Delta\rangle].
    % &= \lambda\E[\langle \nabla F(\x+s\Delta), \Delta\rangle].
\end{align*}
Here the (i) applies fundamental theorem of calculus on $g(s)=F(\x+s\Delta)-F(\x)$ with $g'(s)=\langle \nabla F(\x+s\Delta),\Delta\rangle$ and (ii) uses the following identity for exponential distribution: $\int_t^\infty p(s) ds = \exp(-\lambda t) = p(t)/\lambda$.
\end{proof}

\section{Proof of Theorem \ref{thm:o2nc}}
\label{app:alg}

We restate the formal version of Theorem \ref{thm:o2nc} as follows.
Recall that $S_n = \{\x_t\}_{t\in[n]}$, $\y_n\sim P_n$ where $P_n(\x_t) = \beta^{n-t}\cdot \frac{1-\beta}{1-\beta^n}$, and $\overline \x_n = \sum_{t=1}^n \beta^{n-t}\x_t\cdot \frac{1-\beta}{1-\beta^n}$.

\begin{theorem}
\label{thm:o2nc-formal}
Suppose $F$ is $G$-Lipschitz, $F(\x_0)-\inf F(\x) \le F^*$, and the stochastic gradients satisfy $\E[\nabla f(\x,z)\,|\,\x] = \nabla F(\x)$ and $\E\|\nabla F(\x)-\nabla f(\x,z)\|^2\le \sigma^2$ for all $\x,z$. Define the comparator $\u_n$ and the regret $\regret_n(\u)$ of the regularized losses $\ell_t$ as follows:
\begin{align*}
    & \u_n = -D \cdot \frac{\sum_{t=1}^n \beta^{n-t}\nabla F(\x_t)}{\|\sum_{t=1}^n \beta^{n-t}\nabla F(\x_t)\|},
    & \regret_n(\u) 
    % = \sum_{t=1}^n \ell_t(\Delta_t)
    = \sum_{t=1}^n \langle \beta^{-t}\g_t, \Delta_t - \u \rangle + \calR_t(\Delta_t) - \calR_t(\u).
\end{align*}
Also define the regularizor as $\calR_t(\w) = \frac{\mu_t}{2}\|\w\|^2$ where $\mu_t = \mu\beta^{-t}$, $\mu=\frac{24cD}{\alpha^2}$ and $\alpha=1-\beta$. Then
\begin{align*}
    \E \|\nabla F(\overline\x)\|_c
    &\le \frac{F^*}{DN} + \frac{2G+\sigma}{\alpha N} + \sigma\sqrt{\alpha} + \frac{12c D^2}{\alpha^2} 
    + \frac{1}{DN}\left( \beta^{N+1}\E\regret_N(\u_N) + \alpha \sum_{n=1}^N \beta^n \E\regret_n(\u_n). \right).
\end{align*}
\end{theorem}

\begin{proof}
We start with the change of summation. Note that
\begin{align*}
    \sum_{n=1}^N \sum_{t=1}^n \beta^{n-t}(1-\beta) (F(\x_t) - F(\x_{t-1})) 
    &= \sum_{t=1}^N \left(\sum_{n=t}^N \beta^{n-t}\right) (1-\beta) (F(\x_t) - F(\x_{t-1})) \\
    &= \sum_{t=1}^N (1-\beta^{N-t+1}) (F(\x_t) - F(\x_{t-1})) \\
    &= F(\x_N) - F(\x_0) - \sum_{t=1}^N \beta^{N-t+1}(F(\x_t)-F(\x_{t-1})).
\end{align*}
Upon rearranging and applying the assumption that $F(\x_0)-F(\x_N) \le F(\x_0)-\inf F(\x) \le F^*$, we have
\begin{align}
    -F^* &\le \E\sum_{n=1}^N\sum_{t=1}^n \beta^{n-t}(1-\beta)(F(\x_t)-F(\x_{t-1})) + \E\sum_{t=1}^N \beta^{N-t+1} (F(\x_t) - F(\x_{t-1})).
    \label{eq:exp-1}
\end{align}
% We first bound the first summation, then we will show the second summation is dominated by the first one.

First, we bound the first summation in \eqref{eq:exp-1}. Denote $\calF_t$ as the $\sigma$-algebra of $\x_t$. Note that $\Delta_t\in \calF_t$ and $z_t\not\in\calF_t$, so by the assumption that $\E[\nabla f(\x,z) \,|\, \x] = \nabla F(\x)$,
\begin{align*}
    \E[\g_t \,| \,\calF_t] = \E[\nabla f(\x_t, z_t) \,| \,\calF_t] = \nabla F(\x_t) 
    \implies &\E\langle \nabla F(\x_t),\Delta_t\rangle = \E\langle \g_t,\Delta_t\rangle.
\end{align*}
By Lemma \ref{lem:exp-scaling}, $\E[F(\x_t)-F(\x_{t-1})] = \E\langle \nabla F(\x_t), \Delta_t\rangle$. Upon adding and subtracting, we have
\begin{align*}
    \E[F(\x_t) - F(\x_{t-1})]
    &= \E\langle \nabla F(\x_t) - \g_t + \g_t, \Delta_t - \u_n + \u_n \rangle \\
    &= \E\left[ \langle \nabla F(\x_t), \u_n\rangle\rangle + \langle \nabla F(\x_t)-\g_t, -\u_n\rangle + \langle \g_t,\Delta_t-\u_n\rangle \right].
\end{align*}
Consequently, the first summation in \eqref{eq:exp-1} can be written as
\begin{align}
    % & \E\sum_{n=1}^N\sum_{t=1}^n \beta^{n-t}(1-\beta)(F(\x_t)-F(\x_{t-1})) \\
    % &= 
    \E\sum_{n=1}^N\sum_{t=1}^n \beta^{n-t}(1-\beta) \left( \langle \nabla F(\x_t),\u_n\rangle + \langle \nabla F(\x_t)-\g_t, -\u_n\rangle + \langle \g_t,\Delta_t-\u_n\rangle \right).
    \label{eq:exp-2}
\end{align}
For the first term, upon substituting the definition of $\u_n$, we have
\begin{align*}
    \sum_{t=1}^n \beta^{n-t}(1-\beta) \langle \nabla F(\x_t), \u_n\rangle 
    &= (1-\beta) \left\langle \sum_{t=1}^n \beta^{n-t}\nabla F(\x_t), -D \frac{\sum_{t=1}^n \beta^{n-t} \nabla F(\x_t)}{\|\sum_{t=1}^n \beta^{n-t} \nabla F(\x_t)\|} \right\rangle \\
    &= (1-\beta^n)\cdot -D\left\|\frac{\sum_{t=1}^n \beta^{n-t} \nabla F(\x_t)}{\sum_{t=1}^n \beta^{n-t}}\right\| \\
    &= -D (1-\beta^n) \|\Ex_{\y_n} \nabla F(\y_n) \| 
    \intertext{Since $\|\nabla F(\x_t)\|\le G$ for all $t$, $\|\E_{\y_n} \nabla F(\y_n)\| \le G$ as well. Therefore, we have}
    &\le -D\| \Ex_{\y_n}\nabla F(\y_n)\| + DG\beta^n.
\end{align*}
Since $\beta<1$, $\sum_{n=1}^N \beta^n \le \frac{1}{1-\beta}$. Therefore, upon summing over $n$, the first term in \eqref{eq:exp-2} becomes
\begin{align}
    \E\sum_{n=1}^N\sum_{t=1}^n \beta^{n-t}(1-\beta) \langle \nabla F(\x_t), \u_n\rangle 
    \le \left(-D\sum_{n=1}^N \E \|\Ex_{\y_n} \nabla F(\y_n)\| \right) + \frac{DG}{1-\beta}.
    \label{eq:exp-2a}
\end{align}
For the second term, by Cauchy-Schwarz inequality,
\begin{align*}
    \E \sum_{t=1}^n \beta^{n-t} \langle \nabla F(\x_t)-\g_t, -\u_n\rangle 
    &\le \sqrt{\E \left\| \sum_{t=1}^n \beta^{n-t}(\nabla F(\x_t)-\g_t) \right\|^2 \E\|\u_n\|^2}.
\end{align*}
Since $\E[\nabla F(\x_t)-\g_t \,|\, \calF_t] = 0$, by martingale identity and the assumption that $\E\|\nabla F(\x)-\nabla f(\x,z)\|^2 \le\sigma^2$,
\begin{align*}
    \E \left\| \sum_{t=1}^n \beta^{n-t}(\nabla F(\x_t)-\g_t) \right\|^2
    = \sum_{t=1}^n \E\|\beta^{n-t}(\nabla F(\x_t)-\g_t)\|^2 
    &\le \sum_{t=1}^n \sigma^2 \beta^{2(n-t)}
    \le \frac{\sigma^2}{1-\beta^2}.
\end{align*}
Upon substituting $\|\u_n\|=D$ and $\frac{1}{1-\beta^2}\le \frac{1}{1-\beta}$, the second term in \eqref{eq:exp-2} becomes
\begin{align}
    \E\sum_{n=1}^N\sum_{t=1}^n \beta^{n-t}(1-\beta) \langle \nabla F(\x_t) - \g_t, -\u_n\rangle 
    &\le \sum_{n=1}^N (1-\beta) \cdot \frac{\sigma D}{\sqrt{1-\beta^2}}
    \le \sigma DN\sqrt{1-\beta}.
    \label{eq:exp-2b}
\end{align}
For the third term, upon adding and subtracting $\calR_t$ and substituting the definition of $\regret_n(\u)$, we have
\begin{align}
    &\E\sum_{n=1}^N\sum_{t=1}^n \beta^{n-t}(1-\beta) \langle \g_t,\Delta_t-\u_n\rangle \notag\\
    &= \E\sum_{n=1}^N\sum_{t=1}^n (1-\beta)\beta^n \left( \langle \beta^{-t}\g_t,\Delta_t-\u_n\rangle + \calR_t(\Delta_t) - \calR_t(\u_n) - \calR_t(\Delta_t)+\calR_t(\u_n) \right) \notag\\
    &= \E\sum_{n=1}^N (1-\beta)\beta^n \regret_n(\u_n) + \E\sum_{n=1}^N\sum_{t=1}^n (1-\beta)\beta^n ( -\calR_t(\Delta_t) + \calR_t(\u_n)).
    \label{eq:exp-2c}
\end{align}
Upon substituting \eqref{eq:exp-2a}, \eqref{eq:exp-2b} and \eqref{eq:exp-2c} into \eqref{eq:exp-2}, the first summation in \eqref{eq:exp-1} becomes
\begin{align}
    &\sum_{n=1}^N\sum_{t=1}^n \beta^{n-t}(1-\beta)(F(\x_t)-F(\x_{t-1})) \notag\\
    &\le \left(-D\sum_{n=1}^N \E \|\Ex_{\y_n} \nabla F(\y_n)\| \right) + \frac{DG}{1-\beta} + \sigma DN\sqrt{1-\beta} \notag\\
    &\quad + \E\sum_{n=1}^N (1-\beta)\beta^n \regret_n(\u_n) + \E\sum_{n=1}^N\sum_{t=1}^n (1-\beta)\beta^n ( -\calR_t(\Delta_t) + \calR_t(\u_n)).
    \label{eq:exp-2d}
\end{align}

Next, we consider the second summation in \eqref{eq:exp-1}. Since $\E\|\g_t\| \le \E\|\nabla F(\x_t)\| + \E\|\nabla F(\x_t)-\g_t\| \le G+\sigma$ and $\E\langle \nabla F(\x_t),\Delta_t\rangle = \E\langle \g_t,\Delta_t\rangle$, we have
\begin{align*}
    \E[F(\x_t) - F(\x_{t-1})]
    = \E\langle \nabla F(\x_t),\Delta_t\rangle 
    &= \E\langle \g_t, \Delta_t - \u_N\rangle + \E\langle \g_t, \u_N\rangle \\
    &\le \E\langle \g_t,\Delta_t-\u_N\rangle + D(G+\sigma).
\end{align*}
Following the same argument in \eqref{eq:exp-2c} by adding and subtracting $\calR_t$, the second summation becomes
\begin{align}
    \E\sum_{t=1}^N \beta^{N-t+1} (F(\x_t) - F(\x_{t-1})) 
    &= \E\sum_{t=1}^N \beta^{N+1} \langle \beta^{-t}\g_t,\Delta_t-\u_N\rangle + \beta^{N-t+1} D(G+\sigma) \notag\\
    &\le \beta^{N+1} \E\regret_N(\u_N) + \frac{D(G+\sigma)}{1-\beta} + \E\sum_{t=1}^N \beta^{N+1} (-\calR_t(\Delta_t) + \calR_t(\u_N)).
    \label{eq:exp-3}
\end{align}
Combining \eqref{eq:exp-2d} and \eqref{eq:exp-3} into \eqref{eq:exp-1} gives
\begin{align}
    -F^*
    &\le \left(-D\sum_{n=1}^N \E \|\Ex_{\y_n} \nabla F(\y_n)\| \right) + \frac{DG}{1-\beta} + \sigma DN\sqrt{1-\beta} \notag\\
    &\quad + \E\sum_{n=1}^N (1-\beta)\beta^n \regret_n(\u_n) + \E\sum_{n=1}^N\sum_{t=1}^n (1-\beta)\beta^n ( -\calR_t(\Delta_t) + \calR_t(\u_N)) \notag\\
    &\quad + \beta^{N+1} \E\regret_N(\u_N) + \frac{D(G+\sigma)}{1-\beta} + \E\sum_{t=1}^N \beta^{N+1} (-\calR_t(\Delta_t) + \calR_t(\u_N)).
    \label{eq:exp-4}
\end{align}
As the final step, we simplify the terms involving $\calR_t$. 
% Note that $\|\u_n\|=D$ for all $n$, so $\calR_t(\u_n)$ is independent of $n$ (i.e., $\calR_t(\u_n) = \calR_t(\u_N)$). Consequently, upon apply change of summation, we have
Recall that $\calR_t(\w) = \frac{\mu_t}{2}\|\w\|^2$, so $\calR_t(\u_n)=\frac{\mu_t}{2}D^2$ is independent of $n$. Hence, by change of summation,
\begin{align*}
    &\E\sum_{n=1}^N\sum_{t=1}^n (1-\beta)\beta^n ( -\calR_t(\Delta_t) + \calR_t(\u_n)) + \E\sum_{t=1}^N \beta^{N+1} (-\calR_t(\Delta_t) + \calR_t(\u_N)) \\
    &= \E\sum_{t=1}^N \underbrace{\left(\sum_{n=t}^N \beta^n\right) (1-\beta)}_{=\beta^t - \beta^{N+1}} \left( -\frac{\mu_t}{2}\|\Delta_t\|^2 + \frac{\mu_t}{2}D^2\right) + \E\sum_{t=1}^N \beta^{N+1} \left( -\frac{\mu_t}{2}\|\Delta_t\|^2 + \frac{\mu_t}{2}D^2\right) \\
    &= \E\sum_{t=1}^N \beta^t \left( -\frac{\mu_t}{2}\|\Delta_t\|^2 + \frac{\mu_t}{2}D^2\right) 
    \intertext{Recall Lemma \ref{lem:variance-regularizer} that $\E\sum_{n=1}^N \E_{\y_n}\|\y_n-\overline\x_n\|^2 \le \E\sum_{t=1}^N\frac{12}{(1-\beta)^2}\|\Delta_t\|^2$. Upon substituting $\mu_t=\frac{24cD^2}{(1-\beta)^2}\beta^{-t}$, we have}
    &= \E\sum_{t=1}^N \left( -\frac{12cD}{(1-\beta)^2} \|\Delta_t\|^2 + \frac{12c D^3}{(1-\beta)^2}\right) \\
    &\le \left(-cD \E\sum_{n=1}^N\E_{\y_n}\|\y_n-\overline\x_n\|^2\right) + \frac{12cD^3N}{(1-\beta)^2}.
    \label{eq:exp-4a}
\end{align*}
Substituting this back into \eqref{eq:exp-4} with $\alpha=1-\beta$, we have
\begin{align*}
    -F^* 
    &\le - D \E\left[ \sum_{n=1}^N \|\Ex_{\y_n} \nabla F(\y_n)\| + c\cdot \Ex_{\y_n}\|\y_n-\overline\x_n\|^2 \right] + \frac{DG}{\alpha} + \sigma DN\sqrt{\alpha} + \frac{D(G+\sigma)}{\alpha} + \frac{12cD^3N}{\alpha^2} \\
    &\quad + \beta^{N+1}\E\regret_N(\u_N) + \alpha \sum_{n=1}^N \beta^n \E\regret_n(\u_n).
\end{align*}
By definition of $\|\nabla F(\cdot)\|_c$ defined in Definition \ref{def:stationary-point}, $\|\nabla F(\overline\x_n)\|_c \le \|\Ex_{\y_n} \nabla F(\y_n)\| + c\cdot \Ex_{\y_n}\|\y_n-\overline\x_n\|^2$. Moreover, since $\overline\x$ is uniform over $\overline\x_n$, $\E\|\nabla F(\overline\x)\|_{2,c} = \frac{1}{N}\sum_{n=1}^N \E\|\nabla F(\overline\x_n)\|_{2,c}$ We then conclude the proof by rearranging the equation and dividing both sides by $DN$.
\end{proof}
\section{Proofs in Section \ref{sec:sgdm}}
\label{app:sgdm}

\subsection{Proof of Theorem \ref{thm:OGD-brief}}

Only in this subsection, to be more consistent with the notations in online learning literature, we use $\w$ for weights instead of $\Delta$ as we used in the main text.

To prove the regret bound, we first provide a one-step inequality of OMD with composite loss. Given a convex and continuously differentiable function $\psi$, recall the Bregman divergence of $\psi$ is defined as
\begin{equation*}
    D_\psi(\x,\y) = \psi(\x) - \psi(\y) - \langle \nabla\psi(\y), \y-\x\rangle.
\end{equation*}
Note that $\nabla_\x D_\psi(\x,\y) = \nabla \psi(\x) - \nabla\psi(\y)$. Moreover, as proved in \cite{chen1993convergence}, $D_\psi$ satisfies the following three-point identity:
\begin{align*}
    D_\psi(\z,\x) + D_\psi(\x,\y) - D_\psi(\z,\y) = \langle \nabla\psi(\y) - \nabla\psi(\x), \z-\x\rangle.
\end{align*}

\begin{lemma}
\label{lem:OMD-one-step}
Let $\psi, \phi$ be convex, and define $\w_{t+1} = \argmin_\w \langle \tilde\g_t,\w\rangle + D_\psi(\w,\w_t) + \phi(\w)$. Then for any $\u$,
\begin{align*}
    \langle \tilde\g_t, \w_t-\u\rangle 
    &\le \langle \tilde\g_t, \w_t-\w_{t+1}\rangle + D_\psi(\u,\w_t) - D_\psi(\u,\w_{t+1}) - D_\psi(\w_{t+1},\w_t) + \phi(\u) - \phi(\w_{t+1}).
\end{align*}
\end{lemma}

\begin{proof}
Let $f(\w) = \langle \tilde\g_t,\w\rangle + D_\psi(\w,\w_t) + \phi(\w)$. Since $\psi, \phi$ are convex, so is $f$. Therefore, $\w_{t+1} = \argmin_\w f(\w)$ implies that for all $\u$,
\begin{align*}
    0 
    &\le \langle \nabla f(\w_{t+1}), \u-\w_{t+1}\rangle \\
    &= \langle \tilde\g_t + \nabla\psi(\w_{t+1}) - \nabla\psi(\w_t) + \nabla\phi(\w_{t+1}), \u-\w_{t+1}\rangle \\
    &= \langle\tilde\g_t, \u-\w_t\rangle + \langle\tilde\g_t,\w_t-\w_{t+1}\rangle + \langle \nabla\psi(\w_{t+1})-\nabla\psi(\w_t),\u-\w_{t+1}\rangle + \langle\nabla \phi(\w_{t+1}),\u-\w_{t+1}\rangle.
\end{align*}
Since $\phi$ is convex, $\langle \phi(\w_{t+1}),\u-\w_{t+1}\rangle \le \phi(\u)-\phi(\w_{t+1})$. Moreover, by the three-point identity with $\z=\u, \x=\w_{t+1}, \y=\w_t$, we have
\begin{align*}
    \langle \nabla\psi(\w_t) - \nabla\psi(\w_{t+1}),\u-\w_{t+1}\rangle 
    &= D_\psi(\u,\w_{t+1}) + D_\psi(\w_{t+1},\w) - D_\psi(\u,\w_t).
\end{align*}
Substituting these back and rearranging the inequality then conclude the proof.
\end{proof}

We restate the formal version of Theorem \ref{thm:OGD-brief} as follows.

\begin{theorem}
\label{thm:OGD-formal}
Given a sequence of $\{\tilde\g_t\}_{t=1}^\infty$, a sequence of $\{\eta_t\}_{t=1}^\infty$ such that $0<\eta_{t+1} \le \eta_t$, and a sequence of $\{\mu_t\}_{t=1}^\infty$ such that $\mu_t\ge 0$, let $\calR_t(\w)=\frac{\mu_t}{2}\|\w\|^2$, $\phi_t(\w)=(\frac{1}{\eta_{t+1}}-\frac{1}{\eta_t})\|\w\|^2$, $\w_1=0$ and $\w_t$ updated by
\begin{equation*}
    \w_{t+1} = \argmin_\w \langle \tilde\g_t,\w\rangle + \frac{1}{2\eta_t}\|\w-\w_t\|^2 + \phi_t(\w) + \calR_{t+1}(\w).
\end{equation*}
Then for any $n\in \NN$,
\begin{equation*}
    \sum_{t=1}^n \langle \tilde\g_t, \w_t-\u\rangle + \calR_t(\w_t) - \calR_t(\u) 
    \le \left(\frac{2}{\eta_{n+1}} + \frac{\mu_{n+1}}{2}\right)\|\u\|^2 + \frac{1}{2}\sum_{t=1}^n \eta_t\|\tilde\g_t\|^2.
\end{equation*}
\end{theorem}

\begin{proof}
Denote $\psi_t(\w)=\frac{1}{2\eta_t}\|\w\|^2$. Since $\psi_t,\phi_t,\calR_t$ are all convex and $D_{\psi_t}(\w,\w_t) = \frac{1}{2\eta_t}\|\w-\w_t\|^2$, Lemma \ref{lem:OMD-one-step} holds, which gives
\begin{align*}
    \langle \tilde\g_t, \w_t-\u\rangle 
    &\le \langle \tilde\g_t,\w_t - \w_{t+1}\rangle + D_{\psi_t}(\u,\w_t) - D_{\psi_t}(\u,\w_{t+1}) - D_{\psi_t}(\w_{t+1},\w_t) \\
    &\quad + \phi_t(\u) - \phi_t(\w_{t+1}) + \calR_{t+1}(\u) - \calR_{t+1}(\w_{t+1}).
\end{align*}
Equivalently,
\begin{align}
    \langle \tilde\g_t, \w_t-\u\rangle + \calR_t(\w_t) - \calR_t(\u)
    &\le \langle \tilde\g_t,\w_t - \w_{t+1}\rangle + D_{\psi_t}(\u,\w_t) - D_{\psi_t}(\u,\w_{t+1}) - D_{\psi_t}(\w_{t+1},\w_t) \notag\\
    &\quad + \phi_t(\u) - \phi_t(\w_{t+1}) + \calR_t(\w_t) - \calR_{t+1}(\w_{t+1}) + \calR_{t+1}(\u) - \calR_t(\u).
    \label{eq:ogd-vary-1}
\end{align}
By Young's inequality,
\begin{align*}
    \langle \tilde\g_t,\w_t - \w_{t+1}\rangle - D_{\psi_t}(\w_{t+1},\w_t) 
    &\le \frac{\eta_t}{2}\|\tilde\g_t\|^2 + \frac{1}{2\eta_t}\|\w_{t+1}-\w_t\|^2 - \frac{1}{2\eta_t}\|\w_{t+1}-\w_t\|^2 
    = \frac{\eta_t}{2}\|\tilde\g_t\|^2.
    % \tag{a}
\end{align*}
Next, note that
\begin{align*}
    D_{\psi_t}(\u,\w_t) - D_{\psi_t}(\u,\w_{t+1}) 
    &= D_{\psi_t}(\u,\w_t) - D_{\psi_{t+1}}(\u,\w_{t+1}) + D_{\psi_{t+1}}(\u,\w_{t+1}) - D_{\psi_t}(\u,\w_{t+1}).
    % &= D_{\psi_t}(\u,\w_t) - D_{\psi_{t+1}}(\u,\w_{t+1}) + \left(\frac{1}{2\eta_{t+1}} - \frac{1}{2\eta_t}\right)\|\u-\w_{t+1}\|^2.
    % \tag{b}
\end{align*}
Since $\|\u-\w_{t+1}\|^2 \le 2\|\u\|^2 + 2\|\w_{t+1}\|^2$ and $\frac{1}{\eta_{t+1}}-\frac{1}{\eta_t}\ge 0$,
\begin{align*}
    &D_{\psi_{t+1}}(\u,\w_{t+1}) - D_{\psi_t}(\u,\w_{t+1}) + \phi_t(\u)-\phi_t(\w_{t+1}) \\
    % \left(\frac{1}{2\eta_{t+1}} - \frac{1}{2\eta_t}\right)\|\u-\w_{t+1}\|^2 + \phi_t(\u)-\phi_t(\w_{t+1}) 
    &= \left(\frac{1}{2\eta_{t+1}}-\frac{1}{2\eta_t}\right)\|\u-\w_{t+1}\|^2 + \left(\frac{1}{\eta_{t+1}}-\frac{1}{\eta_t}\right)(\|\u\|^2-\|\w_{t+1}\|^2) 
    \le \left(\frac{2}{\eta_{t+1}} - \frac{2}{\eta_t}\right)\|\u\|^2.
    % \tag{c}
\end{align*}
Upon substituting back into \eqref{eq:ogd-vary-1}, we have
\begin{align*}
    \langle \tilde\g_t, \w_t-\u\rangle + \calR_t(\w_t) - \calR_t(\u)
    &\le \frac{\eta_t}{2}\|\tilde\g_t\|^2 + D_{\psi_t}(\u,\w_t) - D_{\psi_{t+1}}(\u,\w_{t+1}) + \left(\frac{2}{\eta_{t+1}} - \frac{2}{\eta_t}\right)\|\u\|^2 \\
    &\quad + \calR_t(\w_t) - \calR_{t+1}(\w_{t+1}) + \calR_{t+1}(\u) - \calR_t(\u).
\end{align*}
Upon telescoping this one-step inequality, we have
\begin{align*}
    &\sum_{t=1}^n \langle \tilde\g_t, \w_t-\u\rangle + \calR_t(\w_t) - \calR_t(\u) \\
    &\le \left( \sum_{t=1}^n \frac{\eta_t}{2}\|\tilde\g_t\|^2 \right) + D_{\psi_1}(\u,\w_1) - D_{\psi_{n+1}}(\u,\w_{n+1}) + \left(\frac{2}{\eta_{n+1}} - \frac{2}{\eta_1}\right)\|\u\|^2 \\
    &\quad + \calR_1(\w_1) - \calR_{n+1}(\w_{n+1}) + \calR_{n+1}(\u) - \calR_1(\u).
\end{align*}
We then conclude the proof by using $\w_1=0$, $D_{\psi_t}(\u,\w)=\frac{1}{2\eta_t}\|\u-\w\|^2$ and $\calR_n(\w)=\frac{\mu_t}{2}\|\w\|^2$ to simplify
\begin{align*}
    & D_{\psi_1}(\u,\w_1) - D_{\psi_{n+1}}(\u,\w_{n+1}) + \left(\frac{2}{\eta_{n+1}} - \frac{2}{\eta_1}\right)\|\u\|^2 \\
    &\le \frac{1}{2\eta_1}\|\u\|^2 + \left(\frac{2}{\eta_{n+1}} - \frac{2}{\eta_1}\right)\|\u\|^2 \le \frac{2}{\eta_{n+1}}\|\u\|^2
\end{align*}
and $\calR_1(\w_1) - \calR_{n+1}(\w_{n+1}) + \calR_{n+1}(\u) - \calR_1(\u) \le \calR_{n+1}(\u) + \calR_1(\w_1) = \frac{\mu_{n+1}}{2}\|\u\|^2$.
\end{proof}

\subsection{Proof of Theorem \ref{thm:sgdm}}

\SGDM*

\begin{proof}
First, define $D=\frac{F^*}{(G+\sigma)\sqrt{\alpha}N}$, $\mu=\frac{24cD}{\alpha^2}$ and $\eta=\frac{2D\sqrt{\alpha}}{G+\sigma}$. Note that these definitions are equivalent to $\mu=\frac{24F^*c}{(G+D)\alpha^{5/2}N}$ and $\eta=\frac{2F^*}{(G+\sigma)^2N}$ as defined in the theorem. 

Next, note that both Theorem \ref{thm:o2nc-formal} and Theorem \ref{thm:OGD-formal} hold since the explicit update of $\Delta_{t+1}$ is equivalent to
\begin{align*}
    \Delta_{t+1} = \argmin_\Delta \langle\beta^{-t}\g_t, \Delta\rangle + \frac{1}{2\eta_t}\|\Delta-\Delta_t\|^2 + \left(\frac{1}{\eta_{t+1}}-\frac{1}{\eta_t}\right)\|\Delta\|^2 + \frac{\mu_{t+1}}{2}\|\Delta\|^2.
\end{align*}
Also recall that $\regret_n(\u_n) = \sum_{t=1}^n \langle\beta^{-t}\g_t,\Delta_t-\u_n\rangle + \calR_t(\Delta_t) - \calR_t(\u_n)$.
Therefore, upon substituting $\tilde\g_t=\beta^{-t}\g_t$, $\eta_t=\beta^t\eta$, $\mu_t=\beta^{-t}\mu$ and $\|\u_n\|=D$ into Theorem \ref{thm:OGD-formal}, we have
\begin{align*}
    \E\regret_n(\u_n) 
    &\le \left(\frac{2}{\eta_{n+1}}+\frac{\mu_{n+1}}{2}\right)\E\|\u\|^2 + \frac{1}{2}\sum_{t=1}^n \eta_t \E\|\tilde\g_t\|^2 \\
    &= \left(\frac{2}{\eta}+\frac{\mu}{2}\right)D^2\beta^{-(n+1)} + \frac{\eta}{2} \sum_{t=1}^n \beta^{-t}\E\|\g_t\|^2.
    % &= \frac{2D^2}{\beta^{n+1}\eta} + \frac{12cD}{\beta^{n+1}\alpha^2}\cdot D^2 + \frac{1}{2}\sum_{t=1}^n \beta^t\eta \E\|\beta^{-t}\g_t\|^2 \\
    % &\le \frac{2D^2}{\beta^{n+1}\eta} + \frac{12cD^3}{\beta^{n+1}\alpha^2} + \frac{\eta(G^2+\sigma^2)}{2}\cdot \frac{\beta^{-n}}{1-\beta}.
\end{align*}
By Assumption \ref{as:optimization}, $\E\|\g_t\|^2 = \E\|\nabla F(\x_t)\|^2 + \E\|\nabla F(\x_t)-\g_t\|^2 \le G^2+\sigma^2$. Moreover, $\sum_{t=1}^n\beta^{-t} \le \frac{\beta^{-n}}{1-\beta}$. Therefore,
\begin{align*}
    \beta^{n+1}\E\regret_n(\u_n)
    &\le \left(\frac{2}{\eta}+\frac{\mu}{2}\right)D^2 + \frac{\eta(G^2+\sigma^2)}{2\alpha}
    \intertext{Upon substituting $\eta=\frac{2D\sqrt{\alpha}}{G+\sigma}$ (note that $\frac{G^2+\sigma^2}{G+\sigma}\le G+\sigma$) and $\mu=\frac{24cD}{\alpha^2}$, we have}
    &\le \frac{2D(G+\sigma)}{\sqrt{\alpha}} + \frac{12cD^3}{\alpha^2}.
\end{align*}
Consequently, with $\alpha \le \frac{1}{2}$ (so that $\beta^{-1}\le 2$), we have
\begin{align*}
    & \frac{1}{DN}\left(\beta^{N+1}\E\regret_N(\u_N) + \alpha \sum_{n=1}^N \beta^n \E\regret_n(\u_n)\right) \\
    &\le \frac{1+2\alpha N}{DN} \left( \frac{2D(G+\sigma)}{\sqrt{\alpha}} + \frac{12cD^3}{\alpha^2} \right)
    \lesssim \frac{G+\sigma}{N} + \frac{cD^2}{\alpha^2 N} + (G+\sigma)\sqrt{\alpha} + \frac{cD^2}{\alpha}.
\end{align*}
Upon substituting this into the convergence guarantee in Theorem \ref{thm:o2nc-formal}, we have
\begin{align*}
    \E \|\nabla F(\overline\x)\|_c
    % &\le \frac{F^*}{DN} + \frac{2G+\sigma}{\alpha N} + \sigma\sqrt{\alpha} + \frac{12c D^2}{\alpha^2} + \left( \frac{1}{N} + \alpha \right) \left( \frac{2(G+\sigma)}{\sqrt{\alpha}} + \frac{12cD^2}{\alpha^2} \right) \\
    &\le \frac{F^*}{DN} + \frac{2G+\sigma}{\alpha N} + \sigma\sqrt{\alpha} + \frac{12c D^2}{\alpha^2} 
    + \frac{1}{DN}\left( \beta^{N+1}\E\regret_N(\u_N) + \alpha \sum_{n=1}^N \beta^n \E\regret_n(\u_n) \right) \\
    &\lesssim \frac{F^*}{DN} + \frac{G+\sigma}{\alpha N} + (G+\sigma)\sqrt{\alpha} + \frac{cD^2}{\alpha^2} 
    % \tag{Regret is mostly dominated by other terms}
    \intertext{ With $D=\frac{F^*}{(G+\sigma)\sqrt{\alpha}N}$ and $\alpha=\max\{N^{-2/3}, \frac{(F^*)^{4/7}c^{2/7}}{(G+\sigma)^{6/7}N^{4/7}}\}$, we have}
    &\lesssim \frac{G+\sigma}{\alpha N} + (G+\sigma)\sqrt{\alpha} + \frac{(F^*)^2c}{(G+\sigma)^2\alpha^3N^2} 
    \lesssim \frac{G+\sigma}{N^{1/3}} + \frac{(F^*)^{2/7}(G+\sigma)^{4/7}c^{1/7}}{N^{2/7}}.
    \qedhere
\end{align*}
\end{proof}
% \input{appendix/adam_proof}

%%%%%%%%%%%%%%%%%%%%%%%%%%%%%%%%%%%%%%%%%%%%%%%%%%%%%%%%%%%%%%%%%%%%%%%%%%%%%%%
%%%%%%%%%%%%%%%%%%%%%%%%%%%%%%%%%%%%%%%%%%%%%%%%%%%%%%%%%%%%%%%%%%%%%%%%%%%%%%%

\end{document}